\documentclass{article}

\usepackage{times}
\usepackage{graphicx} 
\usepackage{subfigure} 

\usepackage{natbib}

\usepackage{hyperref}
\usepackage{ifthen}

\usepackage[subtle,mathspacing=normal, tracking=normal]{savetrees}
\usepackage{microtype}

\usepackage[accepted]{icml2017}

\icmltitlerunning{Multi-objective Bandits: Optimizing the Generalized Gini Index}

\usepackage{amsmath,amsthm,amssymb,stmaryrd,pifont,wasysym}
\usepackage{algorithm}
\usepackage[noend]{algpseudocode}
\usepackage{bm} 
\usepackage{amsmath}
\usepackage{thmtools,thm-restate}

\newtheorem{lemma}{Lemma}

\newtheorem{proposition}{Proposition}

\newtheorem{rem}{Remark}
\newtheorem{claim}{Claim}

\def\calE{\mathcal{E}}

\def\ext{\mathrm{ext}}
\def\OGDEstep{\mathrm{OGDEstep}}

\newcommand{\bw}{{\bf w}}
\newcommand{\bx}{{\bf x}}

\newcommand{\bs}{{\bf s}}
\newcommand{\by}{{\bf y}}
\newcommand{\bX}{{\bf X}}
\newcommand{\balpha}{{\bm \alpha}}

\newcommand{\bmu}{{\bm \mu}}

\newcommand{\prob}{\mathbb{P}} 
\newcommand{\exptd}{\mathbb{E}}
 
\newcommand{\IND}{\mathbf{1}}

\definecolor{darkgreen}{rgb}{0,0.9,0}
\definecolor{darkred}{rgb}{0.9,0,0}

\definecolor{blue-green}{rgb}{0.0, 0.87, 0.87}
\definecolor{bulgarianrose}{rgb}{0.28, 0.02, 0.03}
\definecolor{calpolypomonagreen}{rgb}{0.12, 0.3, 0.17}
\definecolor{mypink1}{rgb}{0.858, 0.188, 0.478}

\newcommand{\balazs}[1]{{\color{black}#1}}
\newcommand{\paul}[1]{{\color{black}#1}}
\newcommand{\highl}[1]{{\color{black}#1}}
\newcommand{\robi}[1]{{\color{black}#1}}

\newcommand{\bigO}{O}

\newcommand{\argmin}{\operatornamewithlimits{argmin}}

\newcommand\R{\mathbb{R}}   
\newcommand\N{\mathbb{N}}   

\newcommand{\Algo}[1]{\textsc{#1}}

\begin{document} 

\twocolumn[
\icmltitle{Multi-objective Bandits: Optimizing the Generalized Gini Index}




\begin{icmlauthorlist}
\icmlauthor{R\'{o}bert Busa-Fekete}{yahoo}
\icmlauthor{Bal\'azs Sz\"or\'enyi}{rgai,technion}
\icmlauthor{Paul Weng}{jie,jri}
\icmlauthor{Shie Mannor}{technion}
\end{icmlauthorlist}

\icmlaffiliation{rgai}{Research Group on AI, Hungarian Acad. Sci. and Univ. of Szeged, Szeged, Hungary}
\icmlaffiliation{technion}{Technion Institute of Technology, Haifa, Israel}
\icmlaffiliation{jie}{SYSU-CMU JIE, SEIT, SYSU, Guangzhou, P.R. China}
\icmlaffiliation{jri}{SYSU-CMU JRI, Shunde, P.R. China}
\icmlaffiliation{yahoo}{Yahoo Research, New York, NY, USA}

\icmlcorrespondingauthor{Paul Weng}{paul@weng.fr}

\icmlkeywords{Multi-armed bandits, Online convex optimization, Multi-objective}

\vskip 0.3in
]



\printAffiliationsAndNotice{\icmlEqualContribution} 

\begin{abstract} 
	We study the multi-armed bandit (MAB) problem where the agent receives a vectorial feedback that encodes many possibly competing objectives to be optimized. 
    The goal of the agent is to find a policy, which can optimize these objectives simultaneously in a fair way. 
    This multi-objective online optimization problem is formalized by using the Generalized Gini Index (GGI) aggregation function. 
    We propose an online gradient descent algorithm which exploits the convexity of the GGI aggregation function, and controls the exploration in a careful way achieving a distribution-free regret $\tilde{\bigO} (T^{-1/2} )$ with high probability. We test our algorithm on synthetic data as well as on an electric battery control problem where the goal is to trade off the use of the different cells of a battery in order to balance their respective degradation rates.
\end{abstract} 

\section{Introduction}
\label{intro}

The multi-armed bandit (MAB) problem (or bandit problem) refers to an iterative decision making problem in which an agent repeatedly chooses among $K$ options, metaphorically corresponding to pulling one of $K$ arms of a bandit machine. In each round, the agent receives a random payoff, which is a reward or a cost that depends on the arm being selected. The agent's goal is to optimize an evaluation metric, e.g., the \emph{error rate} (expected percentage of times a suboptimal arm is played) or the \emph{cumulative regret} (difference between the sum of payoffs obtained and the (expected) payoffs that could have been obtained by selecting the best arm in each round). In the \emph{stochastic} multi-armed bandit setup, the payoffs are assumed to obey fixed distributions that can vary with the arms but do not change with time. To achieve the desired goal, the agent has to tackle the classical exploration/exploitation dilemma: It has to properly balance the pulling of arms that were found to yield low costs in earlier rounds and the selection of arms that have not yet been tested often enough \cite{AuCeFi02,LaRo85}.

The bandit setup has become the standard modeling framework for many practical applications, such as online advertisement~\cite{Slivkins14} or medical treatment design \cite{Press09} to name a few. In these tasks, the feedback is formulated as a single real value.
However many real-world online learning problems are rather multi-objective, i.e., the feedback consists of a vectorial payoffs.
For example, in our motivating example, namely an electric battery control problem, the learner tries to discover 
a ``best" battery controller, which balances the degradation rates of the battery cells (i.e., components of a battery), among a set of controllers while facing a stochastic power demand. Besides, there are several studies published recently that consider multi-objective sequential decision problem under uncertainty~\citep{DrNo13,RoijersVWD13,MaYaJi13}.

In this paper, we formalize the multi-objective multi-armed bandit setting
where
the feedback received by the agent is in the form of a $D$-dimensional
cost vector. The goal here is to be both efficient, i.e., minimize the cumulative cost for each objective, and fair, i.e., balance the different objectives.
One natural way to ensure this is to try to find a cost vector on the Pareto front that is closest to the origin or to some other ideal point. A generalization of this approach (when using the infinite norm) is the Generalized Gini Index (GGI), a well-known inequality measure in economics \cite{Weymark81}.

GGI is convex, which suggests applying the Online Convex Optimization (OCO) techniques~\citep{Hazan16,Shw12}. However, a direct application of this technique may fail to optimize GGI under noise, because the objective can be only observed with a bias that is induced by the randomness of the cost vectors
and by the fact that the performance is measured by the function value of the average cost instead of the average of the costs' function value. The solution we propose is an online learning algorithm which is based on Online Gradient Descent (\Algo{OGD})~\cite{Zin03,Hazan16} with additional exploration that enables us to control the bias of the objective function.
We also show that its regret is almost optimal: up to a logarithmic factor, it matches the distribution-free lower bound of the stochastic bandit problem~\citep{AuCeFrSc03}, which naturally applies to our setup when the feedback is one-dimensional. 

The paper is organized as follows: after we introduce the formal learning setup, we briefly recall the necessary notions from multi-objective optimization in Section \ref{sec:mo}. Next, in Section \ref{sec:ggi}, GGI is introduced and some of its properties are described. In Sections \ref{sec:optimalpol} and \ref{sec:regret}, we present how to compute the optimal policy for GGI, and  define the regret notion. Section \ref{sec:main} contains our main results where we define our \Algo{OGD}-based algorithm and analyze its regret. In Section \ref{sec:expe}, we test our algorithm and demonstrate its versatility in synthetic and real-world battery-control experiments. In Section 9, we provide a survey of related work, and finally conclude the paper in Section 10.

\section{Formal setup}
\label{sec:setup}

The multi-armed or K-armed bandit problem is specified by real-valued random variables $X_1, \dots, X_K$ associated, respectively, with $K$ arms (that we simply identify by the numbers $1, \ldots , K$). In each time step $t$, the online learner selects one and obtains a random sample of the corresponding distributions.
These samples, which are called costs, are assumed to be independent of all previous actions and costs.\footnote{Our setup is motivated by a practical application where feedback is more naturally formulated in terms of cost. However the stochastic bandit problem is generally based on rewards, which can be easily turned into costs by using the transformation $x \mapsto 1-x$ assuming that the rewards are from $[0,1]$.}
The goal of the learner can be defined in different ways, such as minimizing the sum of costs over time~\cite{LaRo85,AuCeFi02}.

In the \emph{multi-objective} multi-armed bandit (MO-MAB) problem, costs are not 
scalar real values, but real vectors. 
More specifically, a D-objective K-armed bandit problem ($D\ge 2$, $K \ge 2$) is specified by K  real-valued multivariate random variables $\bX_1, \dots , \bX_K$ over $[0,1]^D$. Let $\bmu_{k} = \exptd [\bX_{k} ]$ denote the expected vectorial cost of arm $k$ where $\bmu_{k} = ( \mu_{k,1}, \dots , \mu_{k,D} ) $. Furthermore, $\bmu$ denotes the matrix whose rows are the $\bmu_k$'s.

In each time step the learner can select one of the arms and obtain a sample, which is a cost vector, from the corresponding distribution. 
Samples are also assumed to be independent over time and across the arms, but not necessarily across the components of cost vectors.  At time step $t$, $k_t$ denotes the index of the arm played by the learner and $\bX_{k_t}^{(t)}= (X_{k_t,1}^{(t)}, \dots X_{k_t,D}^{(t)} )$ the resulting payoff.
After playing $t$ time steps, the empirical estimate of the expected cost $\bmu_k$ of the $k$th arm is:


\begin{align}\label{eq:sampmean} 
\widehat{\bmu}_{k}^{(t)} = \frac{1}{T_k(t)}\sum_{\tau=1}^t \bX_{k_\tau}^{(\tau)} \mathbf 1(k_\tau = k)
\end{align}
where all operations are meant elementwise, 
$T_k(t)$ is the number of times the $k$th arm has been played (i.e., $T_k(t) = \sum_{\tau=1}^t \mathbf 1(k_\tau = k)$) and $\mathbf 1( \cdot )$ is the indicator function.


\section{Multi-objective optimization}
\label{sec:mo}

In order to complete the MO-MAB setting, we need to introduce the notion of optimality of the arms. First, we introduce the Pareto dominance relation $\preceq$ defined as follows, for any 
$\bm v, \bm v' \in \mathbb R^D$:
\begin{align}
\bm v \preceq \bm v' \Leftrightarrow \forall d=1, \ldots, D, v_d \le v'_d \enspace.
\end{align}

Let $\mathcal O \subseteq \mathbb R^D$ be a set of $D$-dimension  vectors. 
The {\em Pareto front} of $\mathcal O$, denoted $\mathcal O^*$, is the set of vectors such that:
\begin{align}
\bm v^* \in \mathcal O^* \Leftrightarrow \big( \forall \bm v \in \mathcal O, \bm v \preceq \bm v^* \Rightarrow \bm v = \bm v^* \big) \enspace.
\end{align}

In multi-objective optimization, one usually wants to compute the Pareto front, or search for a particular element of the Pareto front.
In practice, it may be costly (and even infeasible depending on the size of the solution space) to determine all the solutions of the Pareto front.
One may then prefer to directly aim for a particular solution in the Pareto front.
This problem is formalized as a single objective optimization problem, using an \emph{aggregation function}.

An aggregation (or scalarizing) function, which is a non-decreasing 
function $\phi : \mathbb R^D \to \mathbb R$, allows every vector to receive a scalar value to be optimized\footnote{A multivariate function $f : \mathbb R^D \to \mathbb R$ is said to be monotone (non-decreasing) if for all fixed $\bx, \bx^{\prime}\in \R^D$ such that $\bx \preceq \bx^{\prime}$ implies that $f(\bx) \le f( \bx^{\prime})$.  }.
The initial multi-objective problem is then rewritten as follows:
\begin{align}\label{eq:mono}
\min \phi(\bm v) ~~~~~
\text{s.t.} ~~~~~ \bm v \in \mathcal O \enspace .
\end{align}
A solution to this problem yields a particular solution on the Pareto front.
Note that if $\phi$ is not strictly increasing in every component, some care is needed to ensure that the solution of (\ref{eq:mono}) is on the Pareto front.

Different aggregation function can be used depending on the problem at hand, such as sum, weighted sum, $\min$, $\max$, (augmented) weighted Chebyshev norm \cite{SteuerChoo83}, Ordered Weighted Averages (OWA) \cite{Yager88} or Ordered Weighted Regret (OWR) \cite{OgryczakPernyWeng11ADT} and its weighted version \cite{OgryczakPernyWeng13}. In this study, we focus on the Generalized Gini Index (GGI)~\cite{Weymark81}, a special case of OWA.

\section{Generalized Gini Index}\label{sec:ggi}

For a given $n\in \mathbb N$, $[n]$ denotes the set $\{1, 2, \ldots, n\}$. 
The Generalized Gini Index (GGI)~\cite{Weymark81} is defined as follows for a cost vector $\bx = (x_1, \ldots, x_D) \in \mathbb R^D$:
\[
G_{\bw}( \bx ) = \sum_{d=1}^D w_d x_{\sigma(d)} = \bw^\intercal \bx_{\sigma}
\] 
where $\sigma\in \mathbb S_D$, which depends on $\bx$, is the permutation that sorts the components of $\bx$ in a decreasing order, 
$\bx_{\sigma} = ( x_{\sigma(1)}, \cdots, x_{\sigma(D)} )$ is the sorted vector 
and weights $w_i$'s are assumed to be non-increasing, i.e., $w_1 \ge w_2 \ge \ldots \ge w_D$. 
Given this assumption, $G_{\bw}( \bx ) = \max_{\pi\in \mathbb S_D} \bw^\intercal \bx_\pi = \max_{\pi\in \mathbb S_D} \bw_{\pi}^\intercal \bx$ and is therefore a \emph{piecewise-linear convex} function. Figure~\ref{fig:Gini example} illustrates GGI on a bi-objective optimization task.

To better understand GGI, we introduce its formulation in terms of Lorenz vectors.
The Lorenz vector of $\bx$ is the vector $\mathbf L(\bx) = (L_1(\bx), \ldots, L_D(\bx) )$ where $L_d(\bx)$ is the sum of the $d$ smallest components of $\bx$.
Then, GGI can be rewritten as follows: 
\begin{align} \label{eq:ggilorenz}
G_{\bw}( \bx ) = \sum_{d=1}^D w'_d L_d(\bx)
\end{align}
where $\bw' = (w'_1, \ldots, w'_D)$ is the vector defined by $\forall d \in [D], w'_d = w_d - w_{d+1}$ with $w_{D+1} = 0$.
Note that all the components of $\bw'$ are nonnegative as we assume that those of $\bw$ are non-increasing.

\Algo{GGI}\footnote{Note that in this paper GGI is expressed in terms of costs and therefore lower GGI values are preferred.} 
was originally introduced for quantifying the inequality of income distribution in economics. 
It is also known in statistics \cite{BuSz89} as a special case of Weighted Average Ordered Sample statistics, which does not require that weights be non-increasing and is therefore not necessarily convex.
GGI has been characterized by Weymark (\citeyear{Weymark81}). 
	It encodes both efficiency as it is monotone with Pareto dominance and fairness as it is non-increasing with Pigou-Dalton transfers (\citeyear{Pigou12,Dalton20}); 
they are two principles formulating natural requirements, which is an important reason why GGI became a well-established measure of balancedness.
Informally, a Pigou-Dalton transfer amounts to increasing a lower-valued objective while decreasing another higher-valued objective by the same quantity such that the order between the two objectives is not reversed.
The effect of such a transfer is to balance a cost vector.
Formally, GGI satisfies the following fairness property: $\forall \bx \mbox{ such that } x_i < x_j$,
\begin{align*}
\forall \epsilon \in (0, x_j-x_i), G_{\bw}(\bx + \epsilon \bm e_i - \epsilon \bm e_j ) \le G_{\bw}(\bf x )
\end{align*}
where $\bm e_i$ and $\bm e_j$ are two vectors of the canonical basis.
As a consequence, among vectors of equal sum, the best cost vector (w.r.t. GGI) is the one with \emph{equal values in all objectives} if feasible.

\begin{figure}[tb!]
\centering
\includegraphics[scale=.6]{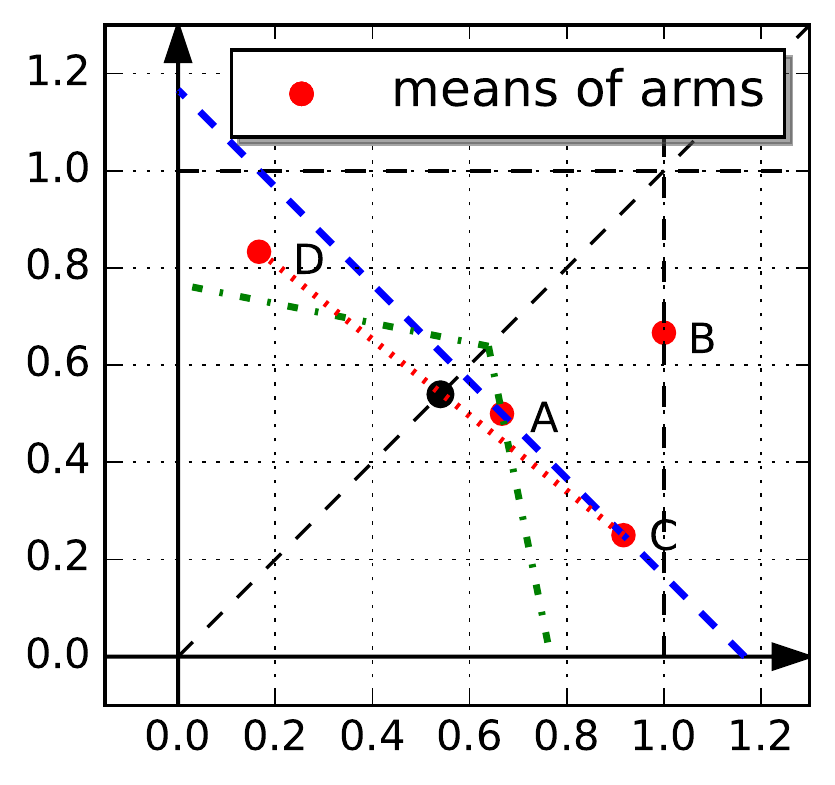}
\caption{Point A is always preferred (w.r.t. GGI) to B due to Pareto dominance; A is always preferred to C due to the Pigou-Dalton transfer principle; depending on the weights of GGI, A may be preferred to D or the other way around; with $\bw=(1, 1/2)$, A is preferred to D (points equivalent to A are on the dashed green line). The optimal mixed solution is the black dot.} 
\label{fig:Gini example}
\end{figure}

The original Gini index can be recovered as a special case of GGI by setting the weights as follows:
\begin{align}\label{eq:giniweights}
\forall d \in [D], w_d = (2(D-d) + 1)/D^2.
\end{align}
This yields a nice graphical interpretation.
For a fixed vector $\bx$, the distribution of costs can be represented as a curve connecting the points $(0, 0)$, $(1/D, L_1(\bx))$, $(2/D, L_2(\bx))$, \ldots, $(1, L_D(\bx))$.
An ideally fair distribution with an identical total cost is given by the straight line connecting $(0, 0)$, $(1/D, L_D(\bx)/D)$, $(2/D, 2L_D(\bx)/D)$, \ldots, $(1, L_D(\bx))$, which equally distributes the total cost over all components.
Then $1-G_\bw(\bx)/\bar x$ with weights (\ref{eq:giniweights}) and $\bar x = \sum x_i/D$ is equal to twice the area between the two curves.

From now on, to simplify the presentation, we focus on \Algo{GGI} 
with strictly decreasing weights in $[0,1]^D$, i.e., $d < d'$ implies $w_d > w_{d'}$. 
This means that GGI is strictly decreasing with Pigou-Dalton transfers and all the components of $\bw'$ are positive.
Based on formulation (\ref{eq:ggilorenz}), \citet{OgryczakSliwinski03} showed that the GGI value of a vector $\bx$ can be obtained by solving a linear program. We shall recall their results and define the linear program-based formulation of GGI.

\begin{proposition}
	\label{prop: dual LP for GGI}
	The \Algo{GGI} score $G_{\bw}( \bx )$ of vector $\bx$ is the optimal value of the following linear program
	\begin{equation*}
	\begin{array}{ll@{}ll}
	\text{minimize}  & \displaystyle \sum_{d=1}^{D} w_{d}^{\prime} \bigg( d r_d & + \sum\limits_{j=1}^{D} b_{j,d} \bigg) \\ 
	\text{subject to}& \displaystyle r_d + b_{j,d} & \ge x_j &\quad \forall j,d \in [D]  \\
	&  b_{j,d} &\ge 0 & \quad \forall j,d \in [D]
	\end{array}
	\end{equation*}
\end{proposition}

\section{Optimal policy}
\label{sec:optimalpol}

In the single objective case, arms are compared in terms of their means, which induce a total ordering over arms. 
In the multi-objective setting, we use the \Algo{GGI} criterion to compare arms. 
One can compute the \Algo{GGI} score of each arm $k$  as $G_{\bw} ( \bmu_k )$ if its vectorial mean $\bmu_k$ is known. 
Then an optimal arm $k^*$ minimizes the \Algo{GGI} score, i.e.,
\[
k^* \in \argmin_{k\in [K] } G_{\bw} ( \bmu_k ) \enspace .
\] 
However, in this work, we also consider {\em mixed} strategies, which can be defined as $\mathcal A = \{ \balpha \in \R^K \vert \sum_{k=1}^K \alpha_k = 1 \,\wedge\, 0 \preceq \balpha \}$, because they may allow to reach lower GGI values than any fixed arm (see Figure~\ref{fig:Gini example}).
A policy parameterized by $\balpha$ chooses arm $k$ with probability $\alpha_k$. 
\robi{An optimal mixed policy can then be obtained as follows:}
\begin{equation}
\label{eq: optimal GGI}
\balpha^* \in \argmin_{\balpha \in A} G_{\bw} \left( \sum_{k=1}^K \alpha_k \bmu_k \right) \enspace .
\end{equation}
In general, $G_{\bw}\left( \sum_{k=1}^K \alpha_k^* \bmu_k \right) \le G_{\bw} ( \bmu_{k^*} ) $, therefore using mixed strategies is justified in our setting.
Based on Proposition~\ref{prop: dual LP for GGI}, if the arms' means were known, $\balpha^{*}$ could be computed by solving the following linear program:
\begin{equation}\label{lp:central}
\begin{array}{ll@{}ll}
\text{minimize}  & \displaystyle\sum_{d=1}^D w^{\prime}_d \left( d r_d + \sum_{j=1}^D b_{j,d} \right) \\
\text{subject to}& \displaystyle r_d + b_{j,d} \ge \highl{\sum_{k=1}^K \alpha_k \mu_{k,j}} & \forall j,d \in [D] \\
& \highl{\sum_{k=1}^K \alpha_k = 1} \qquad \highl{\balpha  \succeq \mathbf{0}}\\
&b_{j,d}  \ge 0 & \forall j,d \in [D] 
\end{array}
\end{equation}

\section{Regret}
\label{sec:regret}

After playing $T$ rounds, the average cost can be written as
\[
\bar{\bX}^{(T)} = \frac{1}{T} \sum_{t=1}^{T} \bX_{k_t}^{(t)} \enspace .
\]
Our goal is to minimize the GGI index of this term.
Accordingly we expect the learner to collect costs so as their average in terms of GGI, that is, $G_{\bw} \left( \bar{\bX}^{(T)} \right)$ should be as small as possible. 
As shown in the previous section, for a given bandit instance with arm means $\bmu = ( \bmu_1, \dots ,\bmu_K )$, the optimal policy $\balpha^*$ achieves $G_{\bw} \left( \sum_{k=1}^K \alpha_k^* \bmu_k  \right) = G_{\bw} \left( \bmu \balpha^{*} \right)$ if the randomness of the costs are not taken into account. 
We consider the performance of the optimal policy as a reference value, and define the regret of the learner as the difference of the GGI of its average cost and the GGI of the optimal policy:
\begin{align}\label{eq:regret}
R^{(T)} = G_{\bw} \left( \bar{\bX}^{(T)} \right)  - G_{\bw} \left( \bmu \balpha^{*} \right) \enspace .
\end{align}
Note that GGI is a continuous function, therefore if the learner follows a policy $\balpha^{(T)}$ that is ``approaching'' $\balpha^{*}$ as $T \rightarrow \infty$, then the regret is vanishing. 

We shall also investigate a slightly different regret notion called {\em pseudo-regret}:
\begin{align}\label{eq:pseudo}
\overline{R}^{(T)} = G_{\bw} \left( \bmu \bar{\balpha}^{(T)} \right)  - G_{\bw} \left( \bmu \balpha^{*} \right) 
\end{align}
where $\bar{\balpha}^{(T)} = \tfrac{1}{T} \sum_{t=1}^{T}\balpha^{(t)}$. We will show that the difference between the regret and pseudo-regret of our algorithm is $\tilde{\bigO} (T^{-1/2})$ with high probability, thus having a high probability regret bound $\tilde{\bigO} (T^{-1/2})$ for one of them implies a regret bound  $\tilde{\bigO} (T^{-1/2})$ for the other one.
\begin{rem} 
	The single objective stochastic multi-armed bandit problem~\cite{AuCeFi02,BuCe12} can be naturally accommodated into our setup with $D=1$ and $w_1 = 1$. 
    In this case, $\balpha^*$ implements the pure strategy that always pulls the optimal arm with the highest mean denoted by $\mu_{k^*}$. 
    Thus $G_{\bw} \left( \bmu \balpha^{*} \right)= \mu_{k^*}$ in this case. 
    Assuming that the learner plays only with pure strategies, the pseudo-regret defined in (\ref{eq:pseudo}) can be written as:
	\[
	\overline{R}^{(T)} = \frac{1}{T} \sum_{t=1}^T \left(\mu_{k_t} - T\mu_{k^*} \right) = \frac{1}{T} \sum_{k=1}^{K} T_{k}(T) \left(\mu_k - \mu_{k^*}\right)
	\]
	which coincides with the single objective pseudo-regret (see for example \citep{AuCeFi02}), apart from the fact that we work with costs instead of rewards. 
    Therefore our notion of multi-objective pseudo-regret can be seen as a generalization of the single objective pseudo-regret.
\end{rem}

\begin{rem} 
Single-objective bandit algorithm can be applied in our multi-objective setting by transforming the multi-variate payoffs $\bX^{(t)}_{k_t}$ into a single real value $G_\bw(\bX^{(t)}_{k_t})$ in every time step $t$. However, in general, this approach fails to optimize GGI as formulated in (\ref{eq: optimal GGI}) due to GGI's non-linearity, even if the optimal policy is a pure strategy. 
Moreover, applying a multi-objective bandit algorithm such as MO-UCB \cite{DrNo13} would be inefficient as they were developed to find all Pareto-optimal arms and then to sample them uniformly at random. This approach may be reasonable to apply only when $\alpha_k^{*} = 1/ \# \mathcal{K}$ where $\mathcal{K} = \{k \in [K]: \alpha_k^*>0 \}$ contains all the Pareto optimal arms, which is clearly not the case for every MO-MAB instance.
\end{rem}



\section{Learning algorithm based on OCO}\label{sec:main}

In this section we propose an online learning algorithm called \Algo{MO-OGDE}, to 
optimize the regret defined in the previous section. Our method exploits the convexity of the GGI operator and formalizes the policy search problem as an online convex optimization problem, which is solved by Online Gradient Descent (OGD) ~\cite{Zin03} algorithm with projection to a gradually expanding truncated probability simplex.   
Then we shall provide a regret analysis of our method. Due to space limitation, the proofs are deferred to the appendix.

\subsection{\Algo{MO-OGDE}}



\begin{algorithm}[t!]
	\caption{\Algo{MO-OGDE}($\delta$)  }
	\label{alg:MO-OGDE}
	\begin{algorithmic}[1]
		\State Pull each arm once
		\State Set $\balpha^{(K+1)} = (1/K, \cdots , 1/K )$
		\For {rounds $t =K+1,K+2, \dots$}
		\State Choose an arm $k_t$ according to $\balpha^{(t)}$
		\State Observe the sample $\bX_{k_t}^{(t)}$ and compute $f^{(t)}$
		\State Set $\eta_t = \frac{\sqrt{2}}{(1-1/\sqrt{K})}\sqrt{\tfrac{\ln (2/\delta)}{t}}$ \label{line:eta}
		\State $\balpha^{(t+1)}=\OGDEstep(\balpha^{(t)}, \eta_t, \nabla f^{(t)} )$
		\EndFor
		\Return $\tfrac{1}{T} \sum_{t=1}^T \balpha^{(t)}$
	\end{algorithmic}
\end{algorithm}

Our objective function to be minimized can be viewed as a function of $\balpha$, i.e., $f (\balpha) = G_{ \bw} ( \bmu  \balpha )$ where the matrix $\bmu = ( \bmu_1, \dots ,\bmu_K )$ contains the means of the arm distributions as its columns. 
Note that the convexity of GGI implies the convexity of $f (\balpha)$. Since we play with mixed strategies, the domain of our optimization problem is the K-dimensional probability simplex $\Delta_K = \{ \balpha \in \R^K \vert \sum_{k=1}^K \alpha_k = 1 \,\wedge\, 0 \preceq \balpha \}$, which is a convex set. 
Then the gradient of $f( \balpha)$  with respect to $\alpha_k$ can be computed as
$
\frac{ \partial f( \balpha ) }{\partial \alpha_k } = \sum_{d=1}^{D} w_{d} \mu_{k,\pi(d)} 
$
 where 
$\pi$ is the permutation that sorts the components of $\bmu\balpha$ in a decreasing order. The means $\bmu_k$'s are not known but they can be estimated based on the costs observed so far as given in (\ref{eq:sampmean}). 
The objective function based on the empirical mean estimates is denoted by $f^{(t)}(\balpha ) = G_{\bw}(\widehat\bmu^{(t)} \balpha )$ where $\widehat{\bmu}^{(t)} = ( \widehat{\bmu}_1^{(t)}, \dots ,\widehat{\bmu}_K^{(t)} )$ contains the empirical estimates for $\bmu_1, \dots, \bmu_K$ in time step $t$ as its columns. 

Our Multi-Objective Online Gradient Descent algorithm with Exploration is defined in Algorithm \ref{alg:MO-OGDE}, which we shall refer to as \Algo{MO-OGDE}. Our algorithm is based on the well-known Online Gradient Descent~\cite{Zin03,Hazan16} that carries out the gradient step and the projection back onto the domain in each round. 
The \Algo{MO-OGDE} algorithm first pulls each arm at once as an initialization step. 
Then in each iteration, it chooses each arm $k$ with probability $\alpha^{(t)}_k$, and it computes $f^{(t)}$ based on the empirical mean estimates. 
Next, it carries out the gradient step based on $ \nabla f^{(t)}( \balpha^{(t)} )$ with  a step size $\eta_t$ as defined in line \ref{line:eta}, and computes the projection $\Pi_{\Delta_K^{\beta}}$ onto the nearest point of the convex set:  
\[ 
\Delta_K^{\beta} = \left\{ \balpha \in \R^K \vert \sum_{k=1}^K \alpha_k = 1 \,\wedge\, \beta/K \preceq \balpha \right\} 
\]
with $\beta = \eta_t$. 
The gradient step and projection are carried out using
\begin{align}
\label{eq:OGDE}
  \OGDEstep(\balpha,\eta,g) =
  \Pi_{\Delta_{K}^{\eta}} \left( \balpha - \eta g ( \balpha )    \right)
\end{align}


The key ingredient of \Algo{MO-OGDE} is the projection step onto the truncated probability simplex $\Delta_K^{\eta_t}$ which ensures that $\alpha_k^{(t)} > \eta_t/K$ for every $k\in [K]$ and $t\in [T]$. 
This forced exploration is indispensable in our setup, since the objective function $f( \balpha)$ depends on the means of the arm distributions, which are not known. 
To control the difference between $ f( \balpha)$ and $f^{(t)} ( \balpha )$, we need ``good'' estimates for $\bmu_1, \dots, \bmu_K$, which are obtained via forced exploration. That is why, the original \Algo{OGD}~\cite{Hazan16} algorithm, in general, fails to optimize GGI in our setting.
Our analysis, presented in the next section, focuses on the interplay between the forced exploration and the bias of the loss functions $f^{(1)}(\balpha), \dots, f^{(T)}(\balpha)$.

\subsection{Regret analysis}

The technical difficulty in optimizing GGI in an online fashion is that, in general, $f^{(t)}( \balpha ) - f( \balpha) ~ (= G_\paul{\bw}(\widehat{\bmu}^{(t)} \balpha) - G_\paul{\bw}(\bmu \balpha))$ is of order $\min_k (T_k(t))^{-1/2}$, which, unless all the arms are sampled a linear number of times, incurs a regret of the same order of magnitude, which is typically too large.
Nevertheless, an optimal policy $\balpha^{*}$ determines a convex combination of several arms in the form of $\bmu \balpha^{*}$, which is the optimal cost vector in terms of GGI given the arm distribution at hand. 
Let us denote ${\mathcal K} = \{ k \in [K] : \alpha^{*}_k > 0\}$. Note that $\#\mathcal{K} \le D$.
Moreover, arms in $[K] \setminus \mathcal{K}$ with an individual GGI lower than arms in $\mathcal K$ do not necessarily participate in the optimal combination. 
Obviously, an algorithm that achieves a $\bigO (T^{-1/2})$ needs to pull the arms in $\#\mathcal{K}$ linear time, and at the same time, estimate the arms in $[K] \setminus \mathcal{K}$ with a certain precision.

The main idea in the approach proposed in this paper is, despite the above remarks, to apply some online convex optimization algorithm on the current estimate $f^{(t)}(\balpha) = G_\paul{\bw}(\widehat{\bmu}^{(t)} \balpha)$ of the objective function $f(\balpha) = G_\paul{\bw}(\bmu \balpha)$, use forced exploration of order $T^{1/2}$, and finally  show that the estimate $f^{(t)}$ of the objective function has error $\tilde{O}(T^{-1/2})$ \emph{along the trajectory} generated by the online convex optimization algorithm.
In particular, we show that
\(
  f(\tfrac{1}{T}\sum_{t=1}^T\balpha^{(t)})
\leq
  \tfrac{1}{T}\sum_{t=1}^T f^{(t)}(\balpha^{(t)})
  +
  \tilde{O}(T^{-1/2})
\).
The intuitive reason for this is that
\(
 \| \tfrac{1}{T}\sum_{t=1}^T \widehat{\bmu}^{(t)}\balpha^{(t)}-\tfrac{1}{T}\sum_{t=1}^T\bmu\balpha^{(t)} \| = \tilde{O}(T^{-1/2})
\),
which is based on the following observation: an arm in $\mathcal{K}$ is either pulled often, thus its mean estimate is then accurate enough, or an arm in $[K]\setminus \mathcal{K}$ is pulled only a few times, nevertheless $\sum_{t=1}^T\alpha_k^{(t)}$ is then small enough to make the poor accuracy of its mean estimate insignificant. 
Below we make this argument formal by proving the following theorem:
\begin{restatable}{theorem}{primetheorem} \label{thm:main}
With probability at least $1- \delta $: 
\begin{align}
& f\Big( \tfrac{1}{T} \sum_{t=1}^T \balpha^{(t)} \Big) - f(\balpha^*) \le 
  2L\sqrt{\tfrac{6D\ln^3(8DKT^2/\delta)}{T}} 
  \notag \enspace ,
\end{align}
for any big enough $T$,
where $L$ is \paul{the} Lipschitz constant of $G_{\bw} (\bx )$.
\end{restatable}
Its proof follows four subsequent steps presented next.

{\bf Step 1 } As a point of departure, we analyze  $\OGDEstep$ in \eqref{eq:OGDE}. In particular, we compute the regret that is commonly used in OCO setup for $f^{(t)} ( \balpha^{(t)})$.
\begin{restatable}{lemma}{primelemma}
\label{lemma:ogde}
	Define $\left\{ \balpha^{(t)} \right\}_{t=1,\dots, T}$ as:
    \begin{align*}
    \balpha^{(1)} &= (1/K, \dots , 1/K )\\
    \balpha^{(t+1)} &= \OGDEstep(\balpha^{(t)}, \eta_t, \nabla f^{(t)})
    \end{align*}
    with $\eta_1, \dots , \eta_T \in [0, 1]$.
    Then the following upper bound is guaranteed for all $T\ge 1$ and for any $\balpha \in \Delta_{K}$:
	\[
	\sum_{t=1}^{T} f^{(t)} (\balpha^{(t)}) - \sum_{t=1}^{T} f^{(t)} (\balpha) \le  \frac{1}{\eta_T} + \frac{G^2+1}{2} \sum_{t=1}^{T}\eta_t
	\]
	where $\sup_{\balpha\in \Delta_K} \| \nabla f^{(t)} (\balpha) \| < G \leq \sqrt{K}D$ for all $t\in [T]$.
\end{restatable}
The proof of Lemma \ref{lemma:ogde} is presented in Appendix \ref{app:a}. If the projection is carried out onto $\Delta_{K}$ according to the \Algo{OGE} algorithm instead of the truncated probability $\Delta_{K}^{\eta_t}$, the regret bound in Lemma \ref{lemma:ogde} would be improved only by a constant factor (see Theorem 3.1 in \cite{Hazan16}). As a side remark, note that  Lemma \ref{lemma:ogde} holds for arbitrary convex function since only the convexity of $f^{(t)}(\balpha)$ is used in the proof. 

\noindent {\bf Step 2 } Next, we show that  $f^{(t)}(\balpha)$ converges to $f (\balpha )$ as fast as $\tilde{O}(T^{-1/2})$ along the trajectory $\left\{ \balpha^{(t)} \right\}_{t=1,\dots, T}$ generated by \Algo{MO-OGDE}. 
\begin{restatable}{proposition}{propfuck}
\label{prop:fuck}
With probability at least $1-2(DT+1)K\delta$, 
\begin{align}
  \left|
    G_\bw\left(\frac{1}{T}\sum_{t=1}^T \bmu\balpha^{(t)}\right) \right.&\left.-~ G_\bw\left(\frac{1}{T}\sum_{t=1}^T \widehat{\bmu}^{(t)}\balpha^{(t)}\right)
  \right|  \notag \\
    & \leq    
    L\sqrt{\frac{6D(1+\ln^2 T)\ln(2/\delta)}{T}} \notag
\enspace .
\end{align}
\end{restatable}
The proof of Proposition \ref{prop:fuck} is deferred to Appendix  \ref{app:fuck}. Proposition \ref{prop:fuck}, combined with the fact that
\begin{align}
  & G_\paul{\bw}\left(\frac{1}{T}\sum_{t=1}^T \widehat{\bmu}^{(t)} \balpha^{(t)}\right)
 \notag \\ & ~~~~~~~~~~~~~~~~ 
  \leq \frac{1}{T}\sum_{t=1}^T G_\paul{\bw}(\widehat{\bmu}^{(t)}\balpha^{(t)})
=
  \frac{1}{T}\sum_{t=1}^T f^{(t)}( \balpha^{(t)}) \notag
\end{align}
where we used the convexity of GGI, and $G_\bw\left(\frac{1}{T}\sum_{t=1}^T \bmu\balpha^{(t)}\right) =  f\left( \frac{1}{T}\sum_{t=1}^T\balpha^{(t)} \right) = f\left( \bar{\balpha}^{(t)} \right)$ implies the following result.
\begin{restatable}{corollary}{corgofalp}
\label{cor:gofalp}
With probability at least $1-2(DT+1)K\delta$,
\begin{align}
  & f\big( \bar{\balpha}^{(t)} \big) 
  \leq
  \frac{1}{T}\sum_{t=1}^T f^{(t)}\big( \balpha^{(t)}\big)
  +
  L\sqrt{\frac{6D(1+\ln^2 T)\ln\tfrac{2}{\delta}}{T}}  \notag
\enspace.
\end{align}
\end{restatable}

{ \bf Step 3 } Next, we provide a regret bound for the pseudo-regret of \Algo{MO-OGDE} by using Lemma \ref{lemma:ogde} and Corollary \ref{cor:gofalp}.
To this end, we introduce some further notations. First of all, let
\[\Delta_K^* = \argmin_{\balpha \in \Delta_K} f(\balpha)\] denote the set of all the minimum points of $f$ over $\Delta_K$.
As we show later in Lemma~\ref{lem: g-star}, $\Delta_K^*$ is a convex polytope, and thus the set $\ext(\Delta_K^*)$ of its extreme points is finite.

\begin{restatable}{proposition}{mainprop}
\label{prop:main}
With probability at least $1- 4DT^2K\delta$: 
\begin{align}
& f\left( \frac{1}{T} \sum_{t=1}^T \balpha^{(t)} \right) - f(\balpha^*)
\le
L\sqrt{\frac{6D(1+\ln^2 T)\ln(2/\delta)}{T}} 
\notag \\
&~~~~~~~~+\frac{1}{T\eta_T} + \frac{KD^2+1}{T} \sum_{t=1}^{T}\eta_t
+
\frac{LK}{T}\sum_{t=1}^T\sqrt{\frac{D\ln(2/\delta)}{2\chi_1(t)}} \notag
\end{align}
where 
\(
\chi_1(t) = \IND(t \le \tau_1) + \IND(t >\tau_1)(ta_0/(2|\ext(\Delta_K^*)|)) 
\)
and 
$\tau_1 = 
\left[(2|\ext(\Delta_K^*)|)\left[ 2+\tfrac{10\sqrt{3}LKD^2}{g^*} \right]\sqrt{2\ln\tfrac{2}{\delta}}\right]^4
$
with 
$a_0 = \min_{\balpha \in \ext(\Delta_K^*)} \min_{k: \alpha_k>0} \alpha_k$ and 
\(
  g^*
=
    \inf_{\balpha \in \Delta_K\setminus\Delta_K^*}
    \;\max_{\balpha^* \in \Delta_K^*}
    \tfrac{f(\balpha)-f(\balpha^*)}{\|\balpha-\balpha^*\|}
$.
\end{restatable}

The proof is deferred to Appendix \ref{app:propmain}. In the proof first we decompose the regret into various terms which can be upper bounded based on Lemma \ref{lemma:ogde} and Corollary \ref{cor:gofalp}. This implies that $\tfrac{1}{T} \sum_{t=1}^{T} \balpha_k^{(t)}$ will get arbitrarily close to some $\balpha_k^{*} \in \Delta_K^*$ 
if $T$ is big enough because, as we show later in Lemma~\ref{lem: g-star}, $g^*$ is strictly positive (see Appendix \ref{app:gstar}).
As a consequence, for a big enough $T> \tau_1$, the difference between the  $f^{(t)}(\balpha^*)$ and $f(\balpha^*)$ is vanishing as fast as $\bigO( T^{-1/2})$ which is crucial for a regret bound of order $T^{-1/2}$. 

{ \bf Step 4 } 
Finally, Theorem  \ref{thm:main} yields from Proposition \ref{prop:main}
by simplifying the right hand side. The calculation is presented in Appendix \ref{app:mainthm}.

\subsection{Regret vs. pseudo-regret}

Next, we upper-bound the difference of regret defined in (\ref{eq:regret}) and the pseudo-regret defined in (\ref{eq:pseudo}). 
To that aim, we first upper-bound the difference of $T_k(t)$ and $\sum_{\tau = 1}^t \alpha_k^{(\tau)}$.
\begin{claim}
For any $t=1, 2, \dots$ and any $k= 1, \dots, K$ it holds that
\(
  \prob
  \left[
    \left|T_k(t) - \sum_{\tau=1}^{t} \alpha_k^{(\tau)}\right| \geq \sqrt{2t\ln(2/\delta)} 
  \right]
\leq
  \delta
\).
\label{clm: Azuma}
\end{claim}
\begin{proof}
As $\prob[k_t=k] = \alpha_k$, it holds that
$T_k(t) - \sum_{\tau=1}^{t} \alpha_k^{(\tau)} = \sum_{\tau=1}^{t} [\IND(k_\tau=k)-\alpha_k^{(\tau)}]$
is a martingale.
Besides, $|\IND(k_\tau=k)- \alpha_k|\leq 1$ by construction.
The claim then follows by Azuma's inequality.
\end{proof}

Based on Claim \ref{clm: Azuma} and Prop.~\ref{prop:fuck}, we upper-bound the difference between the pseudo-regret and regret of \Algo{MO-OGDE}.
\begin{restatable}{corollary}{regretcor}
\label{regret-presudo}
With probability at least $1-\delta$
\begin{align}
  & \vert R^{(T)} - \bar{R}^{(T)}\vert \le L\sqrt{\frac{12D\ln(4(DT+1)/\delta)}{T}} \notag
\end{align}
\end{restatable}
The proof of Corollary \ref{regret-presudo} is deferred to Appendix \ref{app:regret_pseudo}. According to Corollary \ref{regret-presudo}, the difference between the regret and pseudo regret is $\tilde{\bigO} (T^{-1/2})$ with high probability, hence Theorem \ref{thm:main} implies a $\tilde{\bigO} (T^{-1/2})$ bound for the regret of \Algo{MO-OGDE}.


\section{Experiments}\label{sec:expe}


To test our algorithm, we carried out two sets of experiments. In the first we generated synthetic data from multi-objective bandit instances with known parameters. In this way, we could compute the pseudo-regret (\ref{eq:pseudo}) 
and, thus investigate the empirical performance of the algorithms. 
In the second set of experiments, we run our algorithm on a complex multi-objective online optimization problem, namely an electric battery control problem.  
Before presenting those experiments, we introduce another algorithm that will serve as a baseline.

\subsection{A baseline method}

In the previous section we introduced a gradient-based approach that uses the mean estimates to approximate the gradient of the objective function. 
Nevertheless, using the mean estimates, the optimal policy can be directly approximated by solving the the linear program given in (\ref{lp:central}). 
We use the same exploration as \Algo{MO-OGDE}, see line \ref{line:eta} of Algorithm \ref{alg:MO-OGDE}. 
More concretely, the learner solves the following linear program in each time step $t$:
\begin{equation*}
\begin{array}{ll@{}ll}
\text{minimize}  & \displaystyle\sum_{d=1}^D w^{\prime}_d \left( d r_d + \sum_{j=1}^D b_{j,d} \right) \\
\text{subject to}& \displaystyle r_d + b_{j,d} \ge \sum_{k=1}^K \alpha_k  \highl{\widehat{\mu}^{(t)}_{k,j}} ~~& \forall j,d \in [D] \\
& \balpha^T \mathbf{1}  = 1 \\
&\balpha  \ge \highl{\eta_t/K}  \\
&b_{j,d}  \ge 0 \qquad \forall j,d \in [D] 
\end{array}
\end{equation*}
Note that the solution of the learner program above regarding $\balpha$ is in $\Delta_K^{\eta_t}$. 
We refer to this algorithm as \Algo{MO-LP}. 
Note that this approach is computationally expensive, since a linear program needs to be solved at each time step. 
But the policy of each step is always optimal restricted to the truncated simplex $\Delta_{K}^{\eta_t}$ with respect to the mean estimates, unlike the gradient descent method.

\subsection{Synthetic Experiments}

\begin{figure}[t!]
	\begin{center}
		\includegraphics[width=0.494\columnwidth]{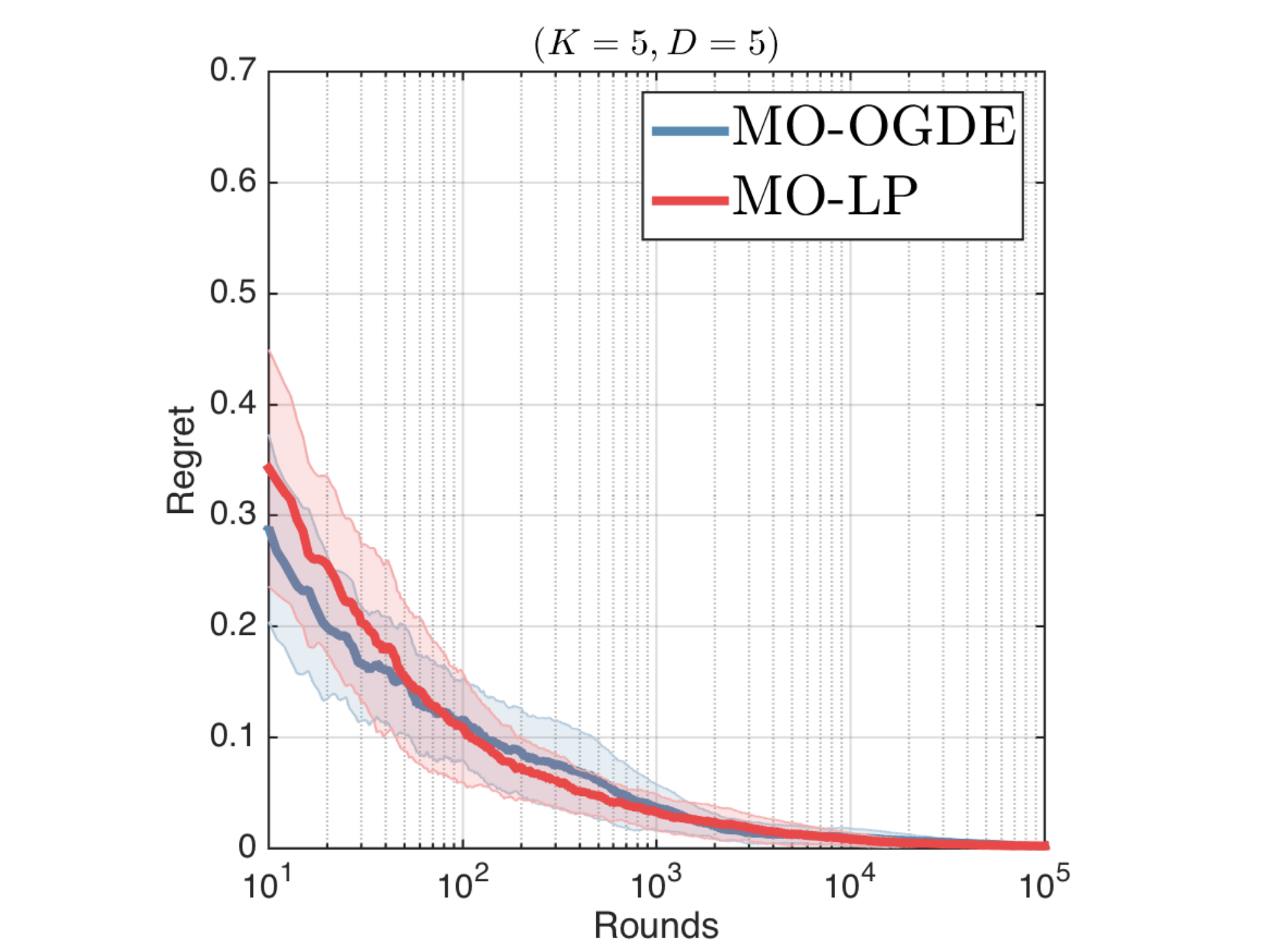}
		\includegraphics[width=0.494\columnwidth]{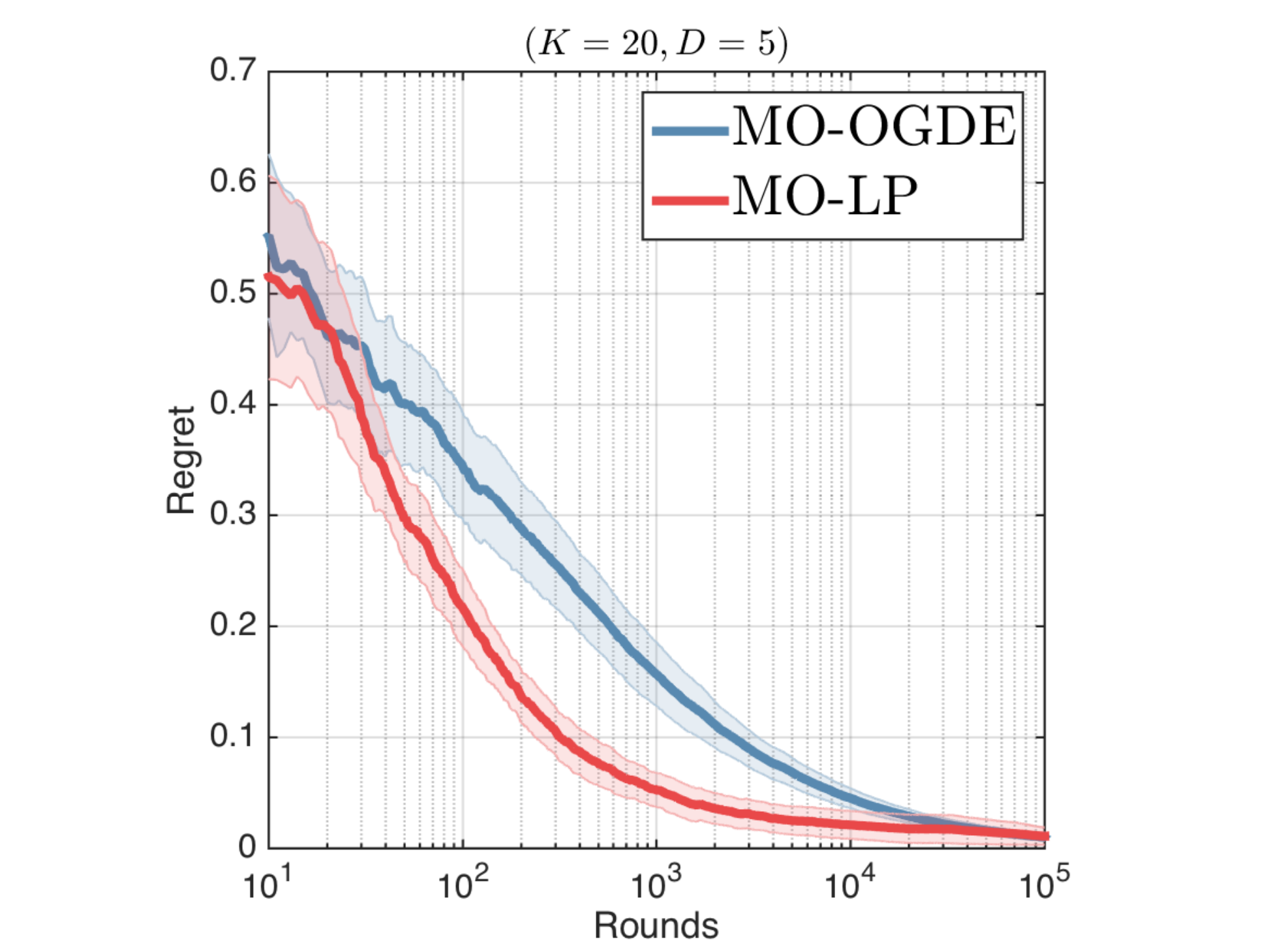}
		\includegraphics[width=0.494\columnwidth]{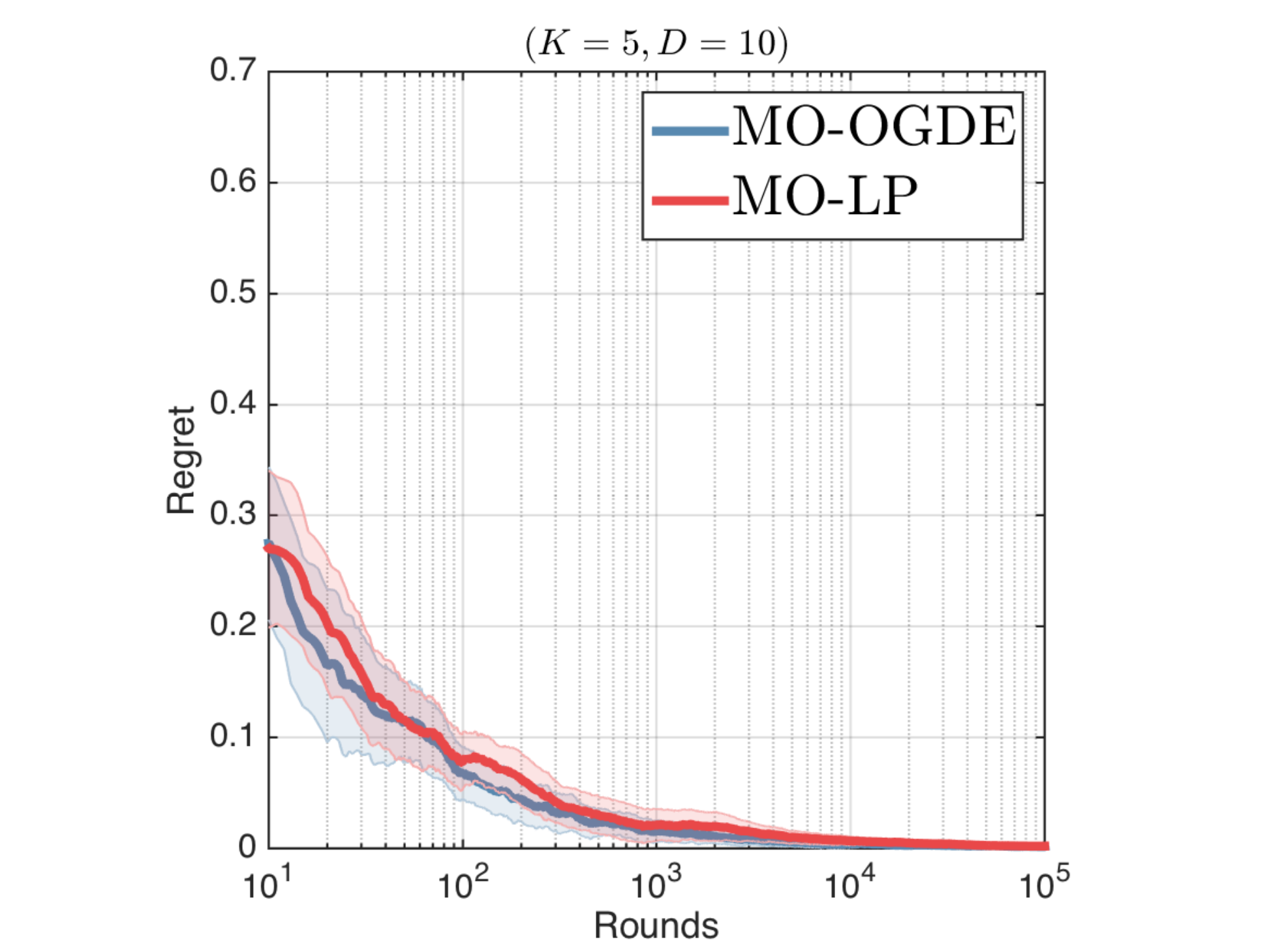}
		\includegraphics[width=0.494\columnwidth]{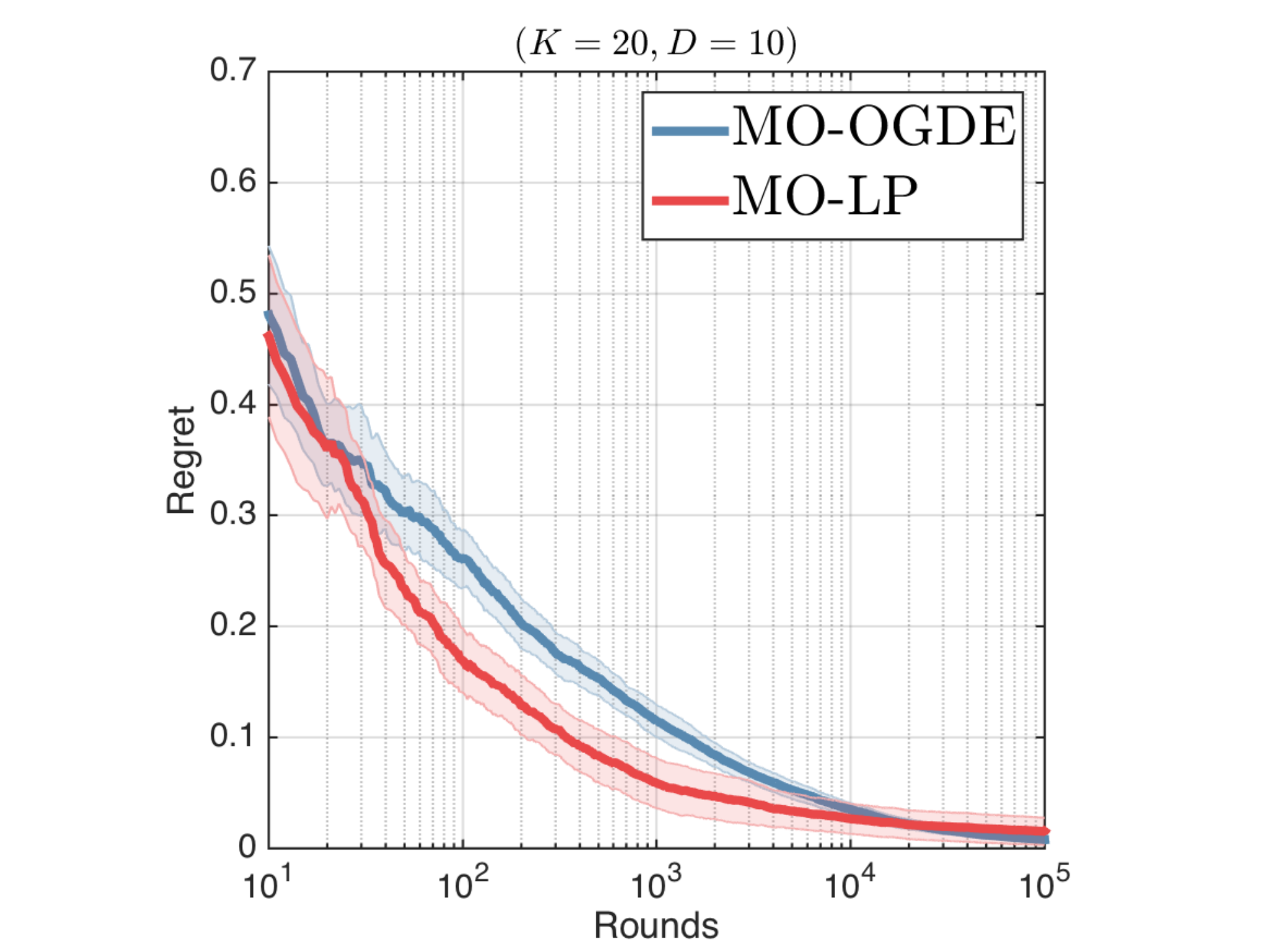}		
		\caption{The regret of the \Algo{MO-LP} and \Algo{MO-OGDE}. The regret and its error is computed based on $100$ repetitions and plotted in terms of the number of rounds. The dimension of the arm distributions was set to $D \in \{5,10\}$, which is indicated in the title of the panels. }
		\label{fig:synthetic}
	\end{center}
	\vskip -0.2in
\end{figure} 

We generated random multi-objective bandit instances for which each component of the multivariate cost distributions obeys Bernoulli distributions with various parameters. The parameters of each Bernoulli distributions are drawn uniformly at random from $[0,1]$ independently from each other. The number of arms $K$ was set to $\{5,20\}$ and the dimension of the cost distribution was taken from $D \in \{5,10\}$. The weight vector $\bw$ of  GGI was set to $w_d= 1/2^{d-1}$. Since the parameters of the bandit instance are known, the regret defined in Section \ref{sec:regret} can be computed. We ran the \Algo{MO-OGDE} and \Algo{MO-LP} algorithms with 100 repetitions. 
The multi-objective bandit instance were regenerated after each run. 
The regrets of the two algorithms, which are averaged out over the repetitions, are plotted in Figure \ref{fig:synthetic} along with the error bars. The results reveal some general trends. 
First, the average regrets of both algorithms converge to zero. 
Second the \Algo{MO-LP} algorithm outperforms the gradient descent algorithm for small number of round, typically $T < 5000$ on the more complex bandit instances ($K=20$). 
This fact might be explained by the fact that the \Algo{MO-LP} solves a linear program for estimating $\alpha^*$ whereas the \Algo{MO-OGDE} minimizes the same objective but using a gradient descent approach with projection, which might achieve slower convergence in terms of regret, nevertheless its computational time is significantly decreased compared to the baseline method. 

\subsection{Battery control task}

\begin{figure}[t!]
	\begin{center}
		\includegraphics[width=0.8\columnwidth]{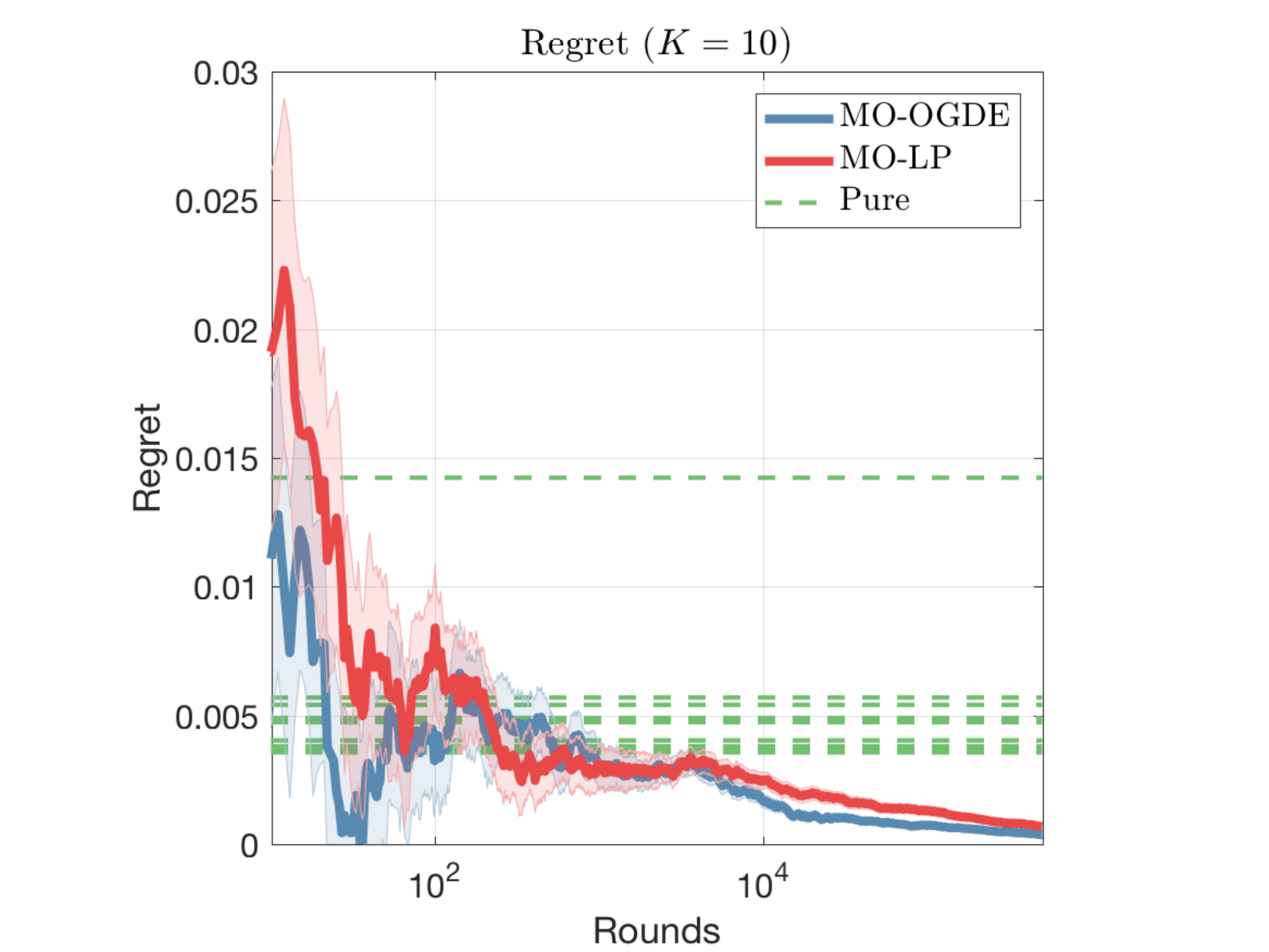}
		\caption{The regret of the \Algo{MO-OGDE} and \Algo{MO-LP} on the battery control task. The regret is averaged over $100$ repetitions and plotted in terms of the number of rounds. The dimension of the arm distributions was $D =12$. }
		\label{fig:battery}
	\end{center}
	\vskip -0.2in
\end{figure} 

We also tried our algorithms on a more realistic domain: the cell balancing problem.
As the performance profile of battery cells, subcomponents of an electric battery, may vary due to small physical and manufacturing differences, efficient balancing of those cells is needed for better performance and longer battery life.
We model this problem as a MO-MAB where the arms are different cell control strategies and the goal is to balance several objectives: state-of-charge (SOC), temperature and aging.
More concretely, the learner loops over the following two steps: (1) she chooses a control strategy for a short duration and (2) she observes its effect on the objectives (due to stochastic electric consumption). 
For the technical details of this
experiments see Appendix~\ref{app:battery}.

We tackled this problem as a GGI optimization problem.
The results (averaged over 100 runs) are presented in Figure~\ref{fig:battery} where we evaluated MO-OGDE vs. MO-LP.
The dashed green lines represent the regrets of playing fixed deterministic arms.
Although MO-OGDE and MO-LP both learn to play a mixed policy that is greatly better than any individual arm, MO-OGDE is computationally much more efficient.

\section{Related work}

The single-objective MAB problem has been intensively studied especially in recent years~\cite{BuCe12}, nevertheless there is only a very limited number of work concerning the multi-objective setting. 
To the best of our best knowledge, \citet{DrNo13} considered first the multi-objective multi-armed problem in a regret optimization framework with a stochastic assumption. 
Their work consists of extending the \Algo{UCB} algorithm~\cite{AuCeFi02} so as to be able to handle multi-dimensional feedback vectors with the goal of determining all arms on the Pareto front. 

\citet{AFFT14} investigated a sequential decision making problem with vectorial feedback. 
In their setup the agent is allowed to choose from a finite set of actions and then it observes the vectorial feedback for each action, thus it is a full information setup unlike our setup. Moreover, the feedback is non-stochastic in their setup, as it is chosen by an adversary. They propose an algorithm that can handle a general class of aggregation functions, such as the set of bounded domain, continuous, Lipschitz and quasi-concave functions.


In the online convex optimization setup with multiple objectives~\citep{MaYaJi13}, the learner's forecast $\bx^{(t)}$ is evaluated in terms of multiple convex loss functions $f_0^{(t)}(\bx), f_1^{(t)} (\bx), \dots, f_K^{(t)}(\bx)$ in each time step $t$. The goal of the learner is then to minimize $\sum_{\tau=1}^t f_0^{(\tau)}(\bx^{(\tau)})$ while keeping the other objectives below some predefined threshold, i.e. $\tfrac{1}{t}\sum_{\tau=1}^t f_i^{(\tau)}(\bx^{(\tau)}) \le \gamma_i$ for all $i \in [K]$. 
Note that, with linear loss functions, the multiple-objective convex optimization setup boils down to linear optimization with stochastic constraints, and thus it can be applied to solve the linear program given in Section \ref{sec:optimalpol} whose solution is the optimal policy in our setup. 
For doing this, however, each linear stochastic constraint needs to be observed, whereas we only assume bandit information.


In the approachability problem~\citep{MannorPS14,MannorTY09, AbernethyBH11}, there are two players, say A and B. Players A and B choose actions from the compact convex sets $\mathcal{X}\subset \R^K$ and $\mathcal{Y}\subset \R^D$, respectively. The feedback to the players is computed as a function $u:\mathcal{X} \times \mathcal{Y} \mapsto \R^p$. A given convex set $S\subset \R^p$ is \emph{known to both players}.
Player A wants to land inside with the cumulative payoffs, i.e., player A's goal is to minimize $\mbox{dist} ( \tfrac{1}{T} \sum_{t=1}^T u( \bx^{(t)}, \by^{(t)}), S)$  where $\bx^{(t)}$ and $\by^{(t)}$ are the actions chosen by player A and B respectively, and $\mbox{dist}( \bs, S ) = \inf_{\bs'\in S} \| \bs'-\bs\|$, 
whereas player B, called adversary, wants to prevent player A to land in set $S$.  In our setup, Player B who generates the cost vectors,  is assumed to be stochastic. 
The set $S$ consists of a single value which is $\bmu \balpha^*$, and $u$ corresponds to $\bmu Ibalpha$, thus $p=D$. What makes our setup essentially different from approachability is that, $S$ is not known to any player. That is why Player A, i.e. the learner, needs to explore the action space which is achieved by forced exploration. 

\vspace{-0.3cm}
\section{Conclusion and future work}
\label{sec:conc}

We introduced a new problem in the context of multi-objective multi-armed bandit (MOMAB). Contrary to most previously proposed approaches in MOMAB, we do not search for the Pareto front, instead we aim for a fair solution, which is important for instance when each objective corresponds to the payoff of a different agent. To encode fairness, we use the Generalized Gini Index (GGI), a well-known criterion developed in economics.
To optimize this criterion, we proposed a gradient-based algorithm that exploits the convexity of GGI. We evaluated our algorithm on two domains and obtained promising experimental results.

Several multi-objective reinforcement learning algorithm have been proposed in the literature~\citep{GaKaSz98,RoijersVWD13}. Most of these methods make use of a simple linear aggregation function. As a future work, it would be interesting to extend our work to the reinforcement learning setting, which would be useful to solve the electric battery control problem even more finely.



\section*{Acknowledgements}

\paul{The authors would like to thank Vikram Bhattacharjee and Orkun Karabasoglu for providing the battery model.}
This research was supported in part by the European Communitys Seventh Framework Programme (FP7/2007-2013)
under grant agreement 306638 (SUPREL).

\bibliography{momab}
\bibliographystyle{icml2017}

\newpage
\appendix
\onecolumn
\allowdisplaybreaks

\begin{center}
	{\LARGE Supplementary material for ``Multi-objective Bandits: Optimizing the Generalized Gini Index''}
\end{center}

For reading convenience, we restate all claims in the appendix.

\section{Lemma \ref{lemma:ogde} }
\label{app:a}

\primelemma*

\begin{proof}
	The proof follows closely the proof of Theorem 3.1 of \citet{Hazan16}, however the projection step is slightly different in our case. First, let us note that $\eta_t\le 1$ for all $t\in [T]$, thus $\Delta_{K}^{\eta_t}$ is never an empty set. Then, for an arbitrary $\balpha \in \Delta_K$, we have
	\begin{align}
	\| \balpha^{(t+1)} - \balpha \|^{2}
	& = \left\| \Pi_{\Delta_K^{\eta_t}} \left( \balpha^{(t)} - \eta_t \nabla_{\balpha} f^{(t)} (\balpha^{(t)})   \right) - \balpha \right\|^{2} \notag \\
	& = \left\|  \Pi_{\Delta_K^{\eta_t}} \left( \balpha^{(t)} - \eta_t \nabla_{\balpha} f^{(t)} (\balpha^{(t)})   \right) - \Pi_{\Delta_K} \left( \balpha^{(t)} - \eta_t \nabla_{\balpha} f^{(t)} (\balpha^{(t)})   \right) \right. \notag \\ & \left. ~~~~~+ \Pi_{\Delta_K} \left( \balpha^{(t)} - \eta_t \nabla_{\balpha} f^{(t)} (\balpha^{(t)})   \right) - \balpha \right\|^{2} \notag \\
	& \le \left\|  \Pi_{\Delta_K^{\eta_t}} \left( \balpha^{(t)} - \eta_t \nabla_{\balpha} f^{(t)} (\balpha^{(t)})   \right) - \Pi_{\Delta_K} \left( \balpha^{(t)} - \eta_t \nabla_{\balpha} f^{(t)} (\balpha^{(t)})   \right) \right\|^{2} \notag \\ & ~~~~~+\left\| \Pi_{\Delta_K} \left( \balpha^{(t)} - \eta_t \nabla_{\balpha} f^{(t)} (\balpha^{(t)})   \right) - \balpha \right\|^{2} \notag \\
	& \le \frac{\eta_t^2}{K}  + \left\| \Pi_{\Delta_K} \left( \balpha^{(t)} - \eta_t \nabla_{\balpha} f^{(t)} (\balpha^{(t)})   \right) - \balpha \right\|^{2} \notag\\
	& \le \frac{\eta_t^2}{K}  + \left\| \balpha^{(t)} - \eta_t \nabla_{\balpha} f^{(t)} (\balpha^{(t)})  - \balpha \right\|^{2} \label{eq:tmp11}
	\end{align}
	where (\ref{eq:tmp11}) follows from the convexity of the set $\Delta_K$. Thus we have
	\begin{align}
	\| \balpha^{(t+1)} - \balpha \|^{2} \le \| \balpha^{(t)} - \balpha \|^{2} + \eta_t^2 \| \nabla_{\balpha} f^{(t)} (\balpha^{(t)}) \|^2 - 2 \eta_t \left( \nabla_{\balpha} f^{(t)} (\balpha^{(t)}) \right)^{\intercal} \left(  \balpha^{(t)} - \balpha \right) + \frac{\eta_t^2}{K} \notag
	\end{align}
	The convexity of $f^{(t)}$ implies that
	\begin{align}
	f^{(t)}( \balpha^{(t)} ) - f^{(t)}( \balpha ) &
	\le 
	\left( \nabla_{\balpha} f^{(t)} (\balpha^{(t)}) \right)^{\intercal} \left(  \balpha^{(t)} - \balpha \right) \notag \\
	&  \le \frac{ \| \balpha^{(t)} - \balpha \|^{2} - \| \balpha^{(t+1)} - \balpha \|^{2}}{2\eta_t} + \frac{\eta_t}{2} G^2  + \frac{\eta_t}{2K}  \label{eq:tmp2121}
	\end{align}
	Therefore the regret can be upper bounded as
	\begin{align}
	\sum_{t=1}^{T} (f^{(t)}( \balpha^{(t)} ) - f^{(t)}( \balpha )) 
	& \le \sum_{t=1}^{T}  \frac{ \| \balpha^{(t)} - \balpha \|^{2} - \| \balpha^{(t+1)} - \balpha \|^{2}}{2\eta_t} + \frac{G^2+1}{2} \sum_{t=1}^{T}\eta_t & \hfill \mbox{ based on (\ref{eq:tmp2121})} \notag \\
	& \le \frac{1}{2}\sum_{t=1}^{T}  \| \balpha^{(t)} - \balpha \|^{2} \left( \frac{1}{\eta_t} - \frac{1}{\eta_{t-1}} \right) + \frac{G^2+1}{2} \sum_{t=1}^{T}\eta_t & \notag \\
	& \le \frac{1}{\eta_T} + \frac{G^2+1}{2} \sum_{t=1}^{T}\eta_t& \label{eq:regret bound with eta} \enspace .
	\end{align}
    \paul{
    We now show that $G \le D\sqrt{K}$.
    By assumption $\bw \in [0, 1]^D$ and $\bmu \in [0, 1]^{D\times K}$, therefore $\bw^\intercal \bmu \in [0, D]^K$.
    As the gradient of $f^{(t)}$ in $\balpha$ is given by $\bw^\intercal \bmu^{(t)}_\sigma$ with $\sigma$ the permutation that orders $\bmu^{(t)} \balpha$ in a decreasing order, we have $G \le D\sqrt{K}$.   
    }
    
\end{proof}

\section{$\tilde{O}(T^{-1/2})$ convergence along the trajectory}
\label{app:fuck}

\propfuck*

\begin{proof}
As the Gini index is $L$-Lipschitz (with $L\le K\sqrt{D}$), we simply need to bound the difference
$\sum_{t=1}^T \widehat{\bmu}^{(t)} \balpha^{(t)} - \sum_{t=1}^T \bmu \balpha^{(t)}$.
The main difficulty is that the $\widehat{\bmu}^{(t)}$ and the $\balpha^{(t)}$ vectors are not independent, so one cannot directly apply the standard concentration inequalities.
In order to obtain the claimed bound, we first divide the above expression into several parts, and deal with these parts separately.

Let $k \in [K]$. 
The division we need is obtained as follows:
\begin{align}
  \sum_{t=1}^T \widehat{\bmu}_k^{(t)} \alpha_k^{(t)}
  -
  \bmu_k \sum_{\tau=1}^T \alpha_k^{(\tau)}
\notag
&=
  \sum_{t=1}^T 
  \left(
    \widehat{\bmu}_k^{(t)}
    \left[
      \sum_{\tau=1}^t \alpha_k^{(\tau)}- \sum_{\tau=1}^{t-1} \alpha_k^{(\tau)} 
    \right]
  \right)
  -
  \bmu_k \sum_{\tau=1}^T \alpha_k^{(\tau)}
\notag
\\
&=
  \sum_{t=1}^{T-1} 
  \left(
    \left[
      \widehat{\bmu}_k^{(t)} - \widehat{\bmu}_k^{(t+1)}
    \right]
    \sum_{\tau=1}^t \alpha_k^{(\tau)} 
  \right)
  +
  \widehat{\bmu}_k^{(T) } \sum_{\tau=1}^T \alpha_k^{(\tau)} 
  -
  \bmu_k \sum_{\tau=1}^T \alpha_k^{(\tau)}
\label{eq: division}
\enspace.
\end{align}

For ease of notation, let $N_k(n) = \argmin\{ \tau \geq 1: T_k(\tau) \geq n \}$, and let $Z_k^n = X_k^{(N_k(n))}$.

The last two terms in (\ref{eq: division}) can be handled as follows:
\begin{align}
  \prob
  &\left[
    \left\|
      \widehat{\bmu}_k^{(T) } \sum_{\tau=1}^T \alpha_k^{(\tau)} 
    -
      \bmu_k \sum_{\tau=1}^T \alpha_k^{(\tau)}
    \right\|
    >
    \sqrt{5DT\ln(2/\delta)}
  \right]
\\
=&
  \prob
  \left[
     \left\|\widehat{\bmu}_k^{(T) } - \bmu_k \right\| \sum_{\tau=1}^T \alpha_k^{(\tau)} 
    >
    \sqrt{5DT\ln(2/\delta)}
  \right]
\notag
\\
\leq&
  \prob
  \left[
     \left\|\widehat{\bmu}_k^{(T) } - \bmu_k \right\| T_k(T	) 
     + \sqrt{2DT\ln(2/\delta)} 
    >
    \sqrt{5DT\ln(2/\delta)}
  \right]
  + \delta
\label{eq: Azuma in part 1}
\\
\leq&
  \prob
  \left[
     \left\|\left(\sum_{n=1}^{T_k(T)} Z_k^n\right) - T_k(T)\bmu_k \right\|
     + \sqrt{2DT\ln(2/\delta)} 
    >
    \sqrt{5DT\ln(2/\delta)}
  \right]
  + \delta
\notag
\\
=&
  \sum_{t=1}^T
  \prob
  \left[
     \left\|\left(\sum_{n=1}^t Z_k^n\right) - t\bmu_k \right\|
     + \sqrt{2DT\ln(2/\delta)}
    >
    \sqrt{5DT\ln(2/\delta)}
      \;;\; T_k(T)=t
  \right]
  + \delta
\notag
\\
\leq&
  \sum_{t=1}^T
  \prob
  \left[
     \sqrt{D(t/2)\ln(2/\delta)}
     + \sqrt{2DT\ln(2/\delta)}
    >
    \sqrt{5DT\ln(2/\delta)}
  \right]
  + (DT+1)\delta
\label{eq: Chernoff in part 1}
\\
=&
  (DT+1)\delta
\label{eq: division bound 1}
\enspace,
\end{align}
where in (\ref{eq: Azuma in part 1}) we used Claim~\ref{clm: Azuma} and the fact that each component of $\widehat{\bmu}_k^{(T)} - \bmu_k$ is bounded by 1 in absolute value,
and in (\ref{eq: Chernoff in part 1}) we used the Chernoff-Hoeffding's inequality with bound $\sqrt{(t/2)\ln(2/\delta)}$ on each of the $D$ components.

Handling the first term requires significantly more work, because the terms within the sum are neither independent nor sub- or supermartingales.
In order to overcome this, we apply a series of rewriting/decoupling iterations.
First of all, note that
\begin{align}
  \sum_{t=1}^{T-1} 
  \left(
    \left[ \widehat{\bmu}_k^{(t)} - \widehat{\bmu}_k^{(t+1)} \right] T_k(t)
  \right)
\notag
=&
  \sum_{n=1}^{T_k(T)-1} 
  \left(
    \left[ \widehat{\bmu}_k^{(N_k(n))} - \widehat{\bmu}_k^{(N_k(n+1))} \right]  n
  \right)
\\
=&
  \sum_{n=1}^{T_k(T)-1} 
  \left(
    \left[
      \left(\widehat{\bmu}_k^{(N_k(n))} - \bmu_k\right)  
    -
      \left(\widehat{\bmu}_k^{(N_k(n+1))} - \bmu_k\right) 
    \right]  n
  \right)
\notag
\\
=&
  \sum_{n=1}^{T_k(T)-1} 
  \left(
    \left[
      \sum_{\tau=1}^n \left(Z_k^\tau - \bmu_k\right) 
    -
      \frac{n}{n+1}\sum_{\tau=1}^{n+1} \left( Z_k^\tau - \bmu_k\right)
    \right]
  \right)
\notag
\\
=&
  \sum_{\tau=1}^{T_k(T)} 
  \left(Z_k^\tau - \bmu_k\right)
  \left(
    \left[ \sum_{n=\tau}^{T_k(T)-1} \frac{1}{n+1} \right]
    - \frac{\tau-1}{\tau}
  \right)
\notag
\enspace,
\end{align}
and thus
\begin{align}
   &\prob
  \left[
    \left\|
      \sum_{n=1}^{T_k(T)-1} 
      \left(
        \left[ \widehat{\bmu}_k^{(N_k(n))} - \widehat{\bmu}_k^{(N_k(n+1))} \right]  n
      \right)
    \right\|
  \geq
    \sqrt{DT(\ln^2 T)\ln(2/\delta)}
  \right]
\notag
\\
&\leq
  \prob
  \left[
      \sum_{\tau=1}^{T_k(T)} 
      \left\|Z_k^\tau - \bmu_k\right\|
      \left|
        \left[ \sum_{n=\tau}^{T_k(T)-1} \frac{1}{n+1} \right]
        - \frac{\tau-1}{\tau}
      \right|
  \geq
    \sqrt{DT(\ln^2 T)\ln(2/\delta)}
  \right]
\notag
\\
&\leq
  \sum_{t=1}^T
  \prob
  \left[
      T_k(T)=t \;;\;
      \sum_{\tau=1}^{T_k(T)} 
      \left\|Z_k^\tau - \bmu_k\right\|
      \left|
        \left[ \sum_{n=\tau}^{T_k(T)-1} \frac{1}{n+1} \right]
        - \frac{\tau-1}{\tau}
      \right|
  \geq
    \sqrt{DT(\ln^2 T)\ln(2/\delta)}
  \right]
\notag
\\
&\leq
  \sum_{t=1}^T
  \prob
  \left[
      \sum_{\tau=1}^{t} 
      \left\|Z_k^\tau - \bmu_k\right\|
      \left|
        \left[ \sum_{n=\tau}^{t-1} \frac{1}{n+1} \right]
        - \frac{\tau-1}{\tau}
      \right|
  \geq
    \sqrt{DT(\ln^2 T)\ln(2/\delta)}
  \right]
\notag
\notag
\\
&\leq
  \sum_{t=1}^T
  D\exp
  \left(
    \frac{-T(\ln^2 T)\ln(2/\delta)}
    {\sum_{\tau=1}^t 
      \left( 
        \left[ \sum_{n=\tau}^{t-1} \frac{1}{n+1} \right] - \frac{\tau-1}{\tau} 
      \right)^2
    }
  \right)
\label{eq: Azuma 3}
\\
&\leq
  \sum_{t=1}^T
  D\exp
  \left(
    \frac{-T(\ln^2 T)\ln(2/\delta)}{t \ln^2 t}
  \right)
\leq
  DT\delta
\enspace,
\label{eq: division middle bound}
\end{align}
where in (\ref{eq: Azuma 3}) we applied the Chernoff-Hoeffding's inequality with bound $\sqrt{T(\ln^2 T)\ln(2/\delta)}$ to each of the $D$ components.

Now, \paul{define $\chi(T)$ as
\begin{align} \label{eq:chi}
\chi(T) = \sqrt{6DT(\ln^2 T)\ln(2/\delta)} \enspace.
\end{align}
We then have:}
\begin{align}
  \prob
  &\left[
  \left\| 
    \sum_{t=1}^{T-1}
    \left(
      \left(\widehat{\bmu}_k^{(t)}-\widehat{\bmu}_k^{(t+1)}\right)\sum_{\tau=1}^t\alpha_k^{(\tau)}
    \right)
  \right\|
  >
  \chi(T)
  \right]
\notag
\\
=&
  \prob
  \left[
  \left\|
    \sum_{n=1}^{T_k(T)-1}
    \left(
      \left(\widehat{\bmu}_k^{(N_k(n))}-\widehat{\bmu}_k^{(N_k(n+1))}\right)
      \sum_{\tau=1}^{N_k(n)}\alpha_k^{(\tau)}
    \right)
  \right\|
  >
  \chi(T)
  \right]
\label{eq: averages are equal}
\\
\leq&
  \prob
  \left[
  \left\| 
    \sum_{n=1}^{T_k(T)-1}
    \left(
      \left(\widehat{\bmu}_k^{(N_k(n))}-\widehat{\bmu}_k^{(N_k(n+1))}\right)
      n
    \right)
  \right\|
  +
  \left\| 
    \sum_{n=1}^{T_k(T)-1}
    \left(
      \left(\widehat{\bmu}_k^{(N_k(n))}-\widehat{\bmu}_k^{(N_k(n+1))}\right)
      \left[n-\sum_{\tau=1}^{N_k(n)} \alpha_k^{(\tau)} \right]
    \right)
  \right\|
  >
  \chi(T)
  \right]
\notag
\\
\leq&
  \prob
  \left[
  \left\| 
    \sum_{n=1}^{T_k(T)-1}
    \left(
      \left(\widehat{\bmu}_k^{(N_k(n))}-\widehat{\bmu}_k^{(N_k(n+1))}\right)
      n
    \right)
  \right\|
  +
  \sum_{n=1}^{T_k(T)-1}
  \left(
    \left\|\widehat{\bmu}_k^{(N_k(n))}-\widehat{\bmu}_k^{(N_k(n+1))}\right\|
    \cdot
    \left|n-\sum_{\tau=1}^{N_k(n)} \alpha_k^{(\tau)} \right|
  \right)
  >
  \chi(T)
  \right]
\notag
\\
\leq&
  \prob
  \left[
  \left\|
    \sum_{n=1}^{T_k(T)-1}
    \left(
      \left(\widehat{\bmu}_k^{(N_k(n))}-\widehat{\bmu}_k^{(N_k(n+1))}\right)
      n
    \right)
  \right\|
  +
  \sqrt{D}
  \sum_{n=1}^{T_k(T)-1}
  \left(
    \frac{2}{n+1}
    \left|n-\sum_{\tau=1}^{N_k(n)} \alpha_k^{(\tau)} \right|
  \right)
  >
  \chi(T)
  \right]
\label{eq: averages are close to each other}
\\
\leq&
  \prob
  \Bigg[
  \sqrt{DT(\ln^2 T)\ln(2/\delta)}
  +
  \sqrt{D}
  \sum_{n=1}^{T_k(T)-1}
  \left(
    \frac{2}{n+1}
    \left|n-\sum_{\tau=1}^{N_k(n)} \alpha_k^{(\tau)} \right|
  \right)
  >
  \chi(T)
  \;,\; \mathcal{E}_1\cap\mathcal{E}_2\Bigg] + \prob[\mathcal{E}_1^c] + \prob[\mathcal{E}_2^c]
\label{eq: mathcal E}
\\
\leq&
  \prob
  \left[
  \sqrt{DT(\ln^2 T)\ln(2/\delta)}
  +
  \sqrt{D}
    \sum_{n=1}^{T_k(T)-1}
    \left(
      \frac{2}{n+1}
      \sqrt{2N_k(n)\ln(2/\delta)}
    \right)
  >
  \chi(T)
  \right]
  +
  (DT+1)\delta
\label{eq: Azuma 4}
\\
\leq&
  \prob
  \left[
  \sqrt{DT(\ln^2 T)\ln(2/\delta)}
  +
  \sqrt{2DT(\ln^2 T)\ln(2/\delta)}
  >
  \chi(T)
  \right]
  +
  (DT+1)\delta
\notag
\\
=&
  (DT+1)\delta
\label{eq: division bound 2}
\end{align}
where (\ref{eq: averages are equal}) follows from the fact that  $\widehat{\bmu}_k^{(t)}=\widehat{\bmu}_k^{(t+1)}$ unless $t+1 = N_k(n)$ for some $n$,
(\ref{eq: averages are close to each other}) follows from the fact that
\(
  (1/\sqrt{D})\left|\widehat{\bmu}_k^{(N_k(n))}-\widehat{\bmu}_k^{(N_k(n+1))}\right|
\leq
  \frac{1}{n+1}\|Z_k^{n+1}\|+\sum_{i=1}^{n}(\tfrac{1}{n}-\frac{1}{n+1})\|Z_k^{i}\|
=
  \frac{2}{n+1}
\),
(\ref{eq: mathcal E}) follows from (\ref{eq: division middle bound}) and
$\mathcal E_1$
denotes the event that 
\(
   \left|T_k(t) - \sum_{t=1}^{t} \alpha_k^{(t)}\right| \leq \sqrt{2t\ln(2/\delta)} 
\)
for all $t=1, \dots, T$ and $\mathcal E_2^c$ denotes the event that
\(
  \left| 
    \sum_{n=1}^{T_k(T)-1}
    \left(
      \left(\widehat{\bmu}_k^{(N_k(n))}-\widehat{\bmu}_k^{(N_k(n+1))}\right)
      n
    \right)
  \right|
\leq
  \sqrt{DT(\ln^2 T)\ln(2/\delta)}
\)
($\mathcal E_1^c$ and $\mathcal E_2^c$ denote the complementary events),
(\ref{eq: Azuma 4}) follows from Claim~\ref{clm: Azuma} and the fact that $T_k(N_k(n)) = n$. 

The claimed bound now follows from (\ref{eq: division}), (\ref{eq: division bound 1}) and (\ref{eq: division bound 2}) because the Gini index is $L$-Lipschitz.
\end{proof}

\corgofalp*


\section{Proof of Proposition \ref{prop:main}}
\label{app:propmain}

In order to ease technicalities, we define events
\begin{align*}
  \calE_\mu
&=
  \left\{
    (\forall 1 \leq t \leq T)
    \sqrt{2T_{k}(t)}\left\|\hat{\bmu}_k^{(t)} - \bmu_k\right\| < \sqrt{D\ln(2\delta)}
  \right\},
\\
  \calE_\alpha
&=
  \left\{
    (\forall 1 \leq t \leq T)
    \left|T_k(t) - \sum_{\tau=1}^{t} \alpha_k^{(\tau)}\right| < \sqrt{2t\ln(2/\delta)} 
  \right\}
\end{align*}
and
\[
  \calE_f
=
  \left\{
    f\big( \bar{\balpha}^{(t)} \big) 
  \leq
  \frac{1}{T}\sum_{t=1}^T f^{(t)}\big( \balpha^{(t)}\big)
  +
  L\sqrt{\tfrac{6D(1+\ln^2 T)\ln(2/\delta)}{T}}
  \right\}
\enspace.
\]

The following technical lemma will also be useful later.

\begin{lemma}
\label{lem: on event T-mu-f-chi}
Let $\balpha^* \in \argmin_{\balpha \in \Delta_K} f(\balpha)$, 
and $\chi = \chi_{\balpha^*} : \N \mapsto \R_+$ be some function and define event
\begin{align*}
  \calE(\chi,\balpha^*)
&=
  \{ 
    (\forall 1 \leq t \leq T)
    \left(\forall k \in [K] \mbox{\ with\ }\alpha_k^*>0\right)
    T_k(t) > \chi(t) 
  \}
\enspace,
\end{align*}
Then, conditioned on the event
\(
    \calE_\mu\cap\calE_f\cap\calE(\chi,\balpha^*)
\), 
\[
      f\left( \frac{1}{T} \sum_{t=1}^T \balpha^{(t)} \right) - f(\balpha^*)
    \leq
      \frac{\zeta^{\chi}(T)}{T}      
\enspace,
\]
where
\begin{align}
  \zeta^{\chi}(\tau)
=
  L\sqrt{6D\tau(1+\ln^2 \tau)\ln(2/\delta)} 
  +
  \frac{1}{\eta_\tau} + \frac{G^2+1}{2} \sum_{t=1}^{\tau}\eta_t
  +
  LK\sum_{t=1}^\tau\sqrt{\frac{D\ln(2/\delta)}{2\chi(t)}} 
\enspace.
\notag
\end{align}
\end{lemma}

\begin{proof}
Throughout the proof condition on the event $\calE_\mu\cap\calE_f\cap\calE(\chi,\balpha^*)$.

First of all, as the generalized Gini index is $L$-Lipschitz, it holds for any $\alpha \in \Delta_K$ that
\[
  \sum_{t=1}^T \left( f^{(t)}(\balpha) - f(\balpha) \right)
\leq
  L \sum_{t=1}^T\left\| \left(\widehat{\bmu}^{(t)} - \bmu\right)\balpha \right\|
\leq
  L K\max_{k \in \mathcal K: \alpha_k > 0} 
    \sum_{t=1}^T\left\|\widehat{\bmu}_k^{(t)} - \bmu_k\right\|
\enspace.
\]
Thus, due to the conditioning on the event $\calE_\mu\cap\calE(\chi,\balpha^*)$,
\begin{align}
  \sum_{t=1}^T \left( f^{(t)}(\balpha^*) - f(\balpha^*) \right)
\leq
  LK\sum_{t=1}^\paul{T}\sqrt{\frac{D\ln(2/\delta)}{2\chi(t)}} 
\label{eq: accurate mu estimates}
\enspace .
\end{align}

It then follows, however, that
\begin{align}
  Tf\left( \frac{1}{T} \sum_{t=1}^T \balpha^{(t)} \right) - Tf(\balpha^*)
&\paul{=}
  Tf\left( \frac{1}{T} \sum_{t=1}^T \balpha^{(t)} \right)
  -
  \left[ \sum_{t=1}^Tf^{(t)}\left(\balpha^{(t)}\right)\right]
  +
  \left[ \sum_{t=1}^Tf^{(t)}\left(\balpha^{(t)}\right)\right]
  -
  Tf(\balpha^*)
\notag
\\
&<
  L\sqrt{6DT(1+\ln^2 T)\ln(2/\delta)} 
  +
  \frac{1}{\eta_T} + \frac{G^2+1}{2} \sum_{t=1}^{T}\eta_t
  +
  \sum_{t=1}^T\left[f^{(t)}(\balpha^*) - f(\balpha^*)\right]
\label{eq: apply even on f}
\\
&<
  \zeta^{\chi}(T)
\enspace,
\label{eq: apply zeta def}
\end{align}
where \eqref{eq: apply even on f} is due to \eqref{eq:regret bound with eta} and the conditioning on $\calE_f$,
and \eqref{eq: apply zeta def} is due to \eqref{eq: accurate mu estimates} and the definition of $\zeta^\chi$.
\end{proof}

Now we are ready to prove the proposition.
For convenience, we recall the statement.

\mainprop*

\begin{proof}
%
Let $\Delta_K^* = \argmin_{\balpha \in \Delta_K}f(\balpha)$ denote the set of optimal solutions of $f$ over $\Delta_K$.
By construction, $\alpha_k^{(t)} \geq \eta_t$ for every $t \geq 1$ and every $k \in [K]$, therefore, setting $\chi_0(\tau) = \max\left\{1,(\sum_{t=1}^\tau \eta_{t}) - \sqrt{2\tau\ln(2/\delta)}\right\}$, it holds for every $\balpha^* \in \Delta_K^*$ that
\begin{align}
  \calE_\alpha \subseteq \calE(\chi_0,\balpha^*)
\label{eq: E-t subseteq E-chi0}
\enspace,
\end{align}
where event $\calE(\chi_0,\balpha^*)$ is defined as in Lemma~\ref{lem: on event T-mu-f-chi}.
Noting that 
\begin{align}
\label{eq: LB for sum of etas}
  \sum_{t=1}^\tau \eta_t
=
  \tfrac{\sqrt{2\ln(2/\delta)}}{1-1/\sqrt{K}} \sum_{t=1}^\tau \tfrac{1}{\sqrt{t}}
\geq
  \tfrac{\sqrt{2\ln(2/\delta)}}{1-1/\sqrt{K}}\int_1^{\tau+1}\tfrac{1}{\sqrt{t}}dt
=
  \tfrac{\sqrt{2\ln(2/\delta)}}{1-1/\sqrt{K}}\left[2\sqrt{\tau+1}-2\right]
\enspace,
\end{align}
it follows that for every $\tau \geq K-1$
\begin{align}
  \chi_0(\tau)
\geq
  \sqrt{2\ln\tfrac{2}{\delta}}
  \left[\tfrac{2\sqrt{\tau+1}-2}{1-1/\sqrt{K}}-\sqrt{\tau}\right]
\geq
  \sqrt{2\ln\tfrac{2}{\delta}}
  \left[\tfrac{(1+1/\sqrt{K})\sqrt{\tau+1}-2}{1-1/\sqrt{K}}\right]
\geq
  \sqrt{2\tau\ln\tfrac{2}{\delta}}
\enspace
,
\label{eq: chi-0 lower bound}
\end{align}
which further implies
\begin{align}
&\zeta^{\chi_0}(\tau)
\notag
\\
&\leq
  L\sqrt{6D\tau(1+\ln^2 \tau)\ln\tfrac{2}{\delta}} 
+
  \tfrac{\left(1-\tfrac{1}{\sqrt{K}}\right)\sqrt{\tau}}{\sqrt{2\ln(2/\delta)}} 
+
  \frac{G^2+1}{2}
  \tfrac{\sqrt{2\ln(2/\delta)}}{1-1/\sqrt{K}}
  \left[2\sqrt{\tau}-1\right]
+
  (K-1)
  +
  LK \sqrt{D}\sqrt[4]{\frac{\ln(2/\delta)}{8}} 
  \sum_{t=K}^\tau \frac{1}{\sqrt[4]{t}}
\label{eq: zeta-chi-0 upper bound 1}
\\
&\leq
L\sqrt{600D\tau^{3/2}\ln\tfrac{2}{\delta}}
+
\sqrt{\tau}\left[ \frac{1}{2\ln\tfrac{2}{\delta}} + (KD^2+1)\sqrt{2\ln\tfrac{2}{\delta}} \right]
+K+
LK \sqrt{D}\sqrt[4]{\tfrac{\ln(2/\delta)}{8}}
\left(\tfrac{4}{3}\tau^{3/4} + K^{-1/4} - \tfrac{4}{3}K^{3/4}\right)
\label{eq: zeta-chi-0 upper bound 2}
\\
&\leq
10LKD^2\sqrt{6\ln(2/\delta)}\tau^{3/4}
\label{eq: bound on zeta-chi}
\enspace,
\end{align}
where \eqref{eq: zeta-chi-0 upper bound 1} follows from \eqref{eq: chi-0 lower bound} and
\begin{align}
\label{eq: UB on sum of etas}
  \sum_{t=1}^\tau \eta_t
=
  \tfrac{\sqrt{2\ln(2/\delta)}}{1-1/\sqrt{K}} \sum_{t=1}^\tau \tfrac{1}{\sqrt{t}}
\leq
  \tfrac{\sqrt{2\ln(2/\delta)}}{1-1/\sqrt{K}}\left[1+\int_1^\tau\tfrac{1}{\sqrt{t}}dt\right]
=
  \tfrac{\sqrt{2\ln(2/\delta)}}{1-1/\sqrt{K}}\left[2\sqrt{\tau}-1\right]
\enspace,
\end{align}
and \eqref{eq: zeta-chi-0 upper bound 2} follows from Lemma~\ref{lemma:ogde}
\begin{align}
  \sum_{t=K}^\tau \frac{1}{\sqrt[4]{t}}
\leq
  K^{-1/4}+\int_K^\tau \frac{1}{\sqrt[4]{t}}dt
\leq
 K^{-1/4}+\tfrac{4}{3}[\tau^{3/4} - K^{3/4}]
\end{align}

Choose $\balpha^* \paul{\in} \argmin_{\balpha \in \Delta_K^*} \|\balpha - \sum_{t=1}^T \balpha^{(t)}\|$.
By Lemma~\ref{lem: g-star}, it holds that
\begin{align}
\label{eq: applying lemma on the diff of alpha-t and alpha-star}
    g^*\max_{k \in \mathcal K}
    \left|
      \left[\frac{1}{T} \sum_{t=1}^T \balpha_k^{(t)} \right] - \balpha_k^*
    \right|
\leq
    f\left( \frac{1}{T} \sum_{t=1}^T \balpha^{(t)} \right) - f(\balpha^*)
\enspace.
\end{align}
Therefore, on event $\calE_\mu \cap \calE_f \cap \calE_\alpha$,
\begin{align}
    \max_{k \in \mathcal K}
    \left|
      \left[\frac{1}{T} \sum_{t=1}^T \balpha_k^{(t)} \right] - \balpha_k^*
    \right|
&\leq
    \frac{\zeta^{\chi_0}(T)}{g^*T}
\label{eq: bound on alphak - alphak-star}
\enspace,
\end{align}
where \eqref{eq: bound on alphak - alphak-star} follows from \eqref{eq: applying lemma on the diff of alpha-t and alpha-star} by Lemma~\ref{lem: on event T-mu-f-chi} due to the conditioning on event $\calE_\mu \cap \calE_f \cap \calE_\alpha$ and recalling \eqref{eq: E-t subseteq E-chi0}.

Now, as $\balpha^* \in \Delta_K^*$, it can be represented as a convex combination of the extreme points of $\Delta_K^*$.
Consider such a representation, and choose $\balpha^+ \in \ext(\Delta_K^*)$ to be the one with maximal coefficient; note that this coefficient must be at least $1/|\ext(\Delta_K^*)|$; that is,
\begin{align}
  \alpha_k^* \geq \tfrac{1}{\ext(|\Delta_K^*|)}\alpha_k^+
\enspace.
\label{eq: alpha-star-k vs alpha-plus-k}
\end{align}
Then, conditioned on event $\calE_\alpha\cap\calE_\mu\cap\calE_f$, for every $k\in [K]$,
\begin{align}
  T_k(\tau)
&\geq
  \sum_{t=1}^\tau\alpha_k^{(t)} - \sqrt{2\tau\ln(2/\delta)}
\label{eq: applying event T-t-delta}
\\
&\geq
  \tau\alpha_k^* - \tfrac{\zeta^{\chi_0}(\tau)}{g^*} - \sqrt{2\tau\ln(2/\delta)}
\label{eq: applying bound on alphak - alphak-star}
\\
&\geq
  \tau\tfrac{\alpha_k^+}{|\ext(\Delta_K^*)|} - \tfrac{\zeta^{\chi_0}(\tau)}{g^*} - \sqrt{2\tau\ln(2/\delta)}
\label{eq_linear lower bound on Tk 1}
\end{align}
where \eqref{eq: applying event T-t-delta} follows from conditioning on event $\calE_\alpha$,
\eqref{eq: applying bound on alphak - alphak-star} follows because of
\eqref{eq: bound on alphak - alphak-star} due to the conditioning on $\calE_\alpha\cap\calE_\mu\cap\calE_f$
,
and \eqref{eq_linear lower bound on Tk 1} follows by \eqref{eq: alpha-star-k vs alpha-plus-k}.
%

Let, now, $a_0 = \min_{\balpha \in \ext(\Delta_K^*)} \min_{k: \alpha_k>0} \alpha_k$, where $\ext(\Delta_K^*)$ denotes the extreme points of $\Delta_K^*$.
Note that $\Delta_K^*$ is a convex polytope due to Lemma~\ref{lem: g-star}, and thus has finite number of extreme points, justifying the $\min$ and implying that $a_0 > 0$.
As $\balpha^+ \in \ext(\Delta_K^*)$, it follows that $\alpha_k^+ \geq a_0$ for every $k$ with $\alpha_k^+ >0$.
Then, by \eqref{eq_linear lower bound on Tk 1}, conditioned on event $\calE_\alpha\cap\calE_\mu\cap\calE_f$, for every $k\in [K]$ with $\alpha_k^+ >0$ and $\tau\geq1$,
\begin{align}
  T_k(\tau)
\geq
  \tau\tfrac{a_0}{|\ext(\Delta_K^*)|} - \tfrac{\zeta^{\chi_0}(\tau)}{g^*} - \sqrt{2\tau\ln(2/\delta)}
\label{eq_linear lower bound on Tk}
\enspace,
\end{align}
implying $T_k(\tau)\geq a_0\tau/(2|\ext(\Delta_K^*)|)$ with
\[
  \tau_1
=
  \left[(2|\ext(\Delta_K^*)|)\left[ 2+\tfrac{10\sqrt{3}LKD^2}{g^*} \right]\sqrt{2\ln\tfrac{2}{\delta}}\right]^4
\geq
  1
  +
  \max\left\{
    \tau \in \N: \tfrac{\tau a_0}{2|\ext(\Delta_K^*)|} 
    <
    \zeta^{\chi_0}(\tau)/g^*
    +\sqrt{2\tau\ln\tfrac@2}{\delta} \right\}
\enspace
,
\]
where the inequality is due to 
\eqref{eq: bound on zeta-chi}.
Consequently,
setting
\[
  \chi_1(\tau) = \IND(\tau \le \tau_1) + \IND(\tau >\tau_1)(\tau a_0/(2|\ext(\Delta_K^*)|))
,
\]
it follows that
\(
  \calE_\alpha\cap\calE_\mu\cap\calE_f \subseteq \calE(\chi_1,\balpha^+)
\).
This completes the proof due to Lemma~\ref{lem: on event T-mu-f-chi} and the upper bound on $G$ from Lemma~\ref{lemma:ogde}, noting that
\(
  \prob[\calE_\alpha^c] \leq \delta
\)
due to Claim~\ref{clm: Azuma} (here, for event $\calE$, we denote its complementer event by $\calE^c$), 
\(
  \prob[\calE_f^c] \leq 2(DT+1)K\delta
\)
due to Corollary~\ref{cor:gofalp},
and
\begin{align}
  \prob[\calE_\mu^c]
&\leq
  \sum_{t=1}^T
  \sum_{\tau=1}^T
  \left[
    \sqrt{2T_{k}(t)}\left\|\hat{\bmu}_k^{(t)} - \bmu_k\right\| \geq \sqrt{D\ln(2\delta)} \;\&\; T_k(t)=\tau
  \right]
\label{eq: bounding P[e-mu] - union bound}
\\
&\leq
  \sum_{t=1}^T
  \sum_{\tau=1}^T
  \left[
    \left\|\sum_{n=1}^\tau Z_k^n - \bmu_k\right\| \geq \sqrt{\frac{D\ln(2\delta)}{2\tau}}
  \right]
\label{eq: bounding P[e-mu] - preparation for Chernoff}
\\
&\leq
  \balazs{2T^2D\delta}
\label{eq: bounding P[e-mu] - Chernoff}
\end{align}
where \eqref{eq: bounding P[e-mu] - union bound} follows due to the union bound, 
$Z_k^n = X_k^{(N_k(n))}$ with $N_k(n) = \argmin\{ \tau \geq 1: T_k(\tau) \geq n \}$ in \eqref{eq: bounding P[e-mu] - preparation for Chernoff}, and
finally \eqref{eq: bounding P[e-mu] - Chernoff} follows due to the Chernoff-Hoeffding bound.
\end{proof}

\section{The proof of $g^*>0$}
\label{app:gstar}

\begin{lemma}
\label{lem: g-star}
The set $\Delta_K^* = \argmin_{\balpha \in \Delta_K}f(\balpha)$ of optimal solutions is a convex polytope.
Additionally, it also holds that
\(
  g^*
=
  \inf_{\balpha \in \Delta_K\setminus\Delta_K^*}
    \;\max_{\balpha^* \in \Delta_K^*}
    \tfrac{f(\balpha)-f(\balpha^*)}{\|\balpha-\balpha^*\|}
\)
is positive.
\end{lemma}

\begin{proof}

Defining, for a permutation $\pi$ over $[D]=\{1,2,\dots,D\}$, the set
\(
  A_\pi
=
  \{ \balpha\in \Delta_K: \forall i,j \in [D], (w_i - w_j)((\bmu\balpha)_{\pi(i)} - (\bmu\balpha)_{\pi(j)}) \geq 0 \}
\),
it follows that $f(\balpha) = G_w(\bmu\balpha)$ is linear over each $A_\pi$; that is,
\(
  f(\balpha)
=
  \mathbf{f}_{\pi}^\intercal \balpha,
\forall \balpha \in A_\pi
\)
for some $\mathbf{f}_{\pi} \in \R^K$ with component $k$ defined as $\sum_{d = 1}^D (w_d\mu_{k,\pi(d)})$.
(Observe that sets associated with different permutations can coincide when $\bw$ has identical components, and that non-coinciding $A_\pi$ and $A_{\pi'}$ overlap each other on some lower dimensional faces.)
It is clear that
\(
  \Delta_K
=
  \cup_{\pi}A_\pi
\enspace.
\)
Additionally, it is also easy to see that $A_\pi$ is a polytope, noting that
\(
  A_\pi
=
  \Delta_K \cap \bigcap_{1 \leq i < j \leq D} \{ \balpha\in \R^K: \balpha^\intercal \mathbf{b}_{i,j,\pi} \geq 0 \}
\),
where vector $\mathbf{b}_{i,j,\pi} \in \R^K$ has component $k$ defined as $[(w_i-w_j)(\bmu_{k,\pi(i)}- \bmu_{k,\pi(j)})]$ for $1 \leq k \leq K$.
Finally, from all the above it also follows that $f$ is a piecewise-linear convex function with linear regions $\{A_\pi\}_\pi$, thus the minimal set $\Delta_K^* = \argmin_{\balpha \in \Delta_K} f(\balpha)$
is a face of one of the $A_\pi$ polytopes.~%
\footnote{See \cite{Boyd2004} for more about piecewise-linear convex functions.}
The first claim of the lemma follows.

Let $\ext(A)$ denote the extreme points of a convex polytope $A$, and
define
\[
  g^+
=
  \min_\pi 
  \min_{\balpha \in \ext(A_\pi)\setminus \Delta_K^*}
  \min_{\balpha^* \in \Delta_K^*} \frac{f(\balpha)-f(\balpha^*)}{\|\balpha-\balpha^*\|}
\enspace.
\]
Note that the definition makes sense and that
\[
0 < g^+ < \infty
\enspace,
\]
because $\Delta_K^*$ is bounded and convex and $\ext(A_\pi)\setminus \Delta_K^*$ is a finite set with no elements from $\Delta_K^*$.

Now, choose some permutation $\pi$ and some $\balpha \in A_\pi \setminus \Delta_K^*$.
Then $\balpha$ can be written as a convex combination $\balpha = \sum_{\balpha' \in \ext(A_\pi)}  \balpha'\omega(\balpha')$ for some weight vector $\omega: \ext(A_\pi) \to [0,1]$ with $\sum_{\balpha' \in \ext(A_\pi)}  \omega(\balpha') = 1$.
(This form is possibly non-unique.)
Let $N^* = \{\balpha \in \ext(A_\pi) \cap \Delta_K^*: \omega(\balpha)>0\}$, and let 
$\balpha^* = \tfrac{1}{\sum_{\balpha' \in N^*}\omega(\balpha') }\sum_{\balpha' \in N^*} \balpha'\omega(\balpha')$ when $N^* \neq \emptyset$, otherwise chose $\balpha^* \in \Delta_K^*$ in an arbitrary fashion.
In either case,
\begin{align}
\label{eq: structure of a-star}
  \sum_{\balpha' \in N^*}  \omega(\balpha')(\balpha^* - \balpha') = 0
\enspace.
\end{align}
Then, since $f$ is linear over $A_\pi$,
\begin{align}
  \frac{f(\balpha)-f(\balpha^*)}{\|\balpha^* - \balpha\|}
=&
  \frac
  {\sum_{\balpha' \not\in N^*} \omega(\balpha') (f(\balpha') - f(\balpha^*))}
  {
    \left\|
      \left(
        \sum_{\balpha' \not\in N^*} \omega(\balpha')(\balpha^* - \balpha')
      \right) 
    - 
        \left(
          \sum_{\balpha' \in N^*} \omega(\balpha')(\balpha^* - \balpha')
        \right)
    \right\|
  }
\notag\\
\geq&
  \frac
  {\sum_{\balpha' \not\in N^*} \omega(\balpha') (f(\balpha') - f(\balpha^*))}
  {
    \sum_{\balpha' \not\in N^*} \omega(\balpha')
    \left\|\balpha^*-\balpha'\right\|
  }
\label{eq: detaching optimals}
\\
\geq&
  \frac
  {\sum_{\balpha' \not\in N^*} \omega(\balpha') (f(\balpha') - f(\balpha^*))}
  {\sum_{\balpha' \not\in N^*} \omega(\balpha') (f(\balpha') - f(\balpha^*))/g^+}
= g^+
\notag
\enspace.
\end{align}
where \eqref{eq: detaching optimals} follows because of \eqref{eq: structure of a-star} and the triangle inequality.
As $\alpha^* \in \Delta_K^*$ we obtain that $g^* \geq g^+$, completing the proof of the second claim of the lemma.
\end{proof}

\section{Proof of  Theorem \ref{thm:main}}
\label{app:mainthm}

\primetheorem*

\begin{proof}
First we upper-bound the following sum that appears in the last term of Proposition \ref{prop:main}:
\begin{align}
\sum_{t=1}^T\frac{1}{\sqrt{\chi_1(t)}} 
& = 
 \tau_1 + 
 \sqrt{\frac{2|\ext(\Delta_K^*)|}{a_0}}
 \sum_{\tau = \tau_1+1}^T  \frac{1}{\sqrt{\tau}} \notag \\
& \leq 
 \tau_1 + 
 \sqrt{\frac{2|\ext(\Delta_K^*)|}{a_0}}
 \left[ \sqrt{\tau_1} + \int_{\tau_1}^T \frac{1}{\sqrt{\tau}}d\tau \right]
 \notag \\
& \leq 
 \sqrt{\frac{2|\ext(\Delta_K^*)|}{a_0}}
 \left[ \tau_1 + 2\sqrt{T} -\sqrt{\tau_1}  \right]
 \notag \\
& \leq
 \sqrt{\frac{2|\ext(\Delta_K^*)|}{a_0}}
  \left[
    (2|\ext(\Delta_K^*)|)
    \left[ 2+\tfrac{10\sqrt{3}LKD^2}{g^*} \right]
    \sqrt{2\ln\tfrac{2}{\delta}}
  \right]^4  
  +
  \sqrt{\frac{8|\ext(\Delta_K^*)|}{a_0}}\sqrt{T}
 \notag \\
& \leq
  \frac{(2|\ext(\Delta_K^*)|)^{9/2}}{\sqrt{a_0}}
  \left[ 2+\tfrac{10\sqrt{3}LKD^2}{g^*} \right]^4
  \left(2\ln\tfrac{2}{\delta}\right)^2
  +
  \sqrt{\frac{8|\ext(\Delta_K^*)|}{a_0}}\sqrt{T}
\notag
\end{align}
According to Proposition \ref{prop:main}, using $\delta/4DKT^2$ in place of $\delta$ and recalling \eqref{eq: UB on sum of etas},
\begin{align}
f&\left( \frac{1}{T} \sum_{t=1}^T \balpha^{(t)} \right) - f(\balpha^*) 
\nopagebreak
\\
\le&
L\sqrt{\frac{6D(1+\ln^2 T)\ln\tfrac{8DKT^2}{\delta}}{T}} 
+
\frac{1-1/\sqrt{K}}{\sqrt{2T\ln(2/\delta)}}
+ 
[KD^2+1] 
\frac{\sqrt{2\ln(2/\delta)}}{1-1/\sqrt{K}}\left[\tfrac{2}{\sqrt{T}}-\tfrac{1}{T}\right]
\notag
\\
&+
LK \frac{\sqrt{D\ln\tfrac{8DKT^2}{\delta}}}{\sqrt{2}T}
\frac{(2|\ext(\Delta_K^*)|)^{9/2}}{\sqrt{a_0}}
  \left[ 2+\tfrac{10\sqrt{3}LKD^2}{g^*} \right]^4
  \left(2\ln\tfrac{2}{\delta}\right)^2
+
2LK \sqrt{\frac{D|\ext(\Delta_K^*)|\ln\tfrac{8DKT^2}{\delta}}{a_0T}}
\notag \\
\le&
  2L\sqrt{\frac{6D\ln\tfrac{8DKT^2}{\delta}}{T}} 
  \max
  \left\{
    \ln(2T)
  ,\;
    1+4(KD)^{3/2}
    +
    2K\tfrac{|\ext(\Delta_K^*)|^{9/2}}{\sqrt{a_0}}
    \left[ 2+\tfrac{10\sqrt{3}LKD}{g^*} \right]^4
    (2\ln(2/\delta))
  \right\}
\notag
\end{align}
which completes the proof.
\end{proof}

\section{Regret vs. pseudo regret}
\label{app:regret_pseudo}
\regretcor* 

\begin{proof}
The average cost can be written as
\[
\bar{\bX}^{(T)} 
= \frac{1}{T} \sum_{k=1}^{K} T_k(T) \frac{1}{T_k(T)} \sum_{t=1}^{T} \mathbf{1}(k_t = k) \bX_k^{(t)} 
= \frac{1}{T} \sum_{k=1}^{K} T_k(T) \hat{\bmu}_k^{(T)}
\]
According to Claim 1, with probability at least $1-\delta/2$, it holds
\begin{align}
\| \bar{\bX}^{(T)} - \hat{\bmu}^{(T)} \bar{\balpha}^{(T)} \| & = 
\left\| \frac{1}{T} \sum_{k=1}^{K} T_k(T) \hat{\bmu}_k^{(T)} - \frac{1}{T} \sum_{k=1}^{K} \hat{\bmu}_{k}^{(T)} \sum_{t=1}^{T}\alpha_k^{(t)} \right\| \notag \\
& = \left\| \frac{1}{T} \sum_{k=1}^{K} \hat{\bmu}_k^{(T)} \left( T_k(T) - \sum_{t=1}^{T}\alpha_k^{(t)} \right) \right\| \notag \\
&\le K\sqrt{\frac{2D\ln (4K/\delta)}{T}} \notag
\end{align}
and since GGI is $L$-Lipschitz, with probability at least $1-\delta/2$,
\begin{align}\label{eq:aux1}
\vert G_\bw(\bar{\bX}^{(T)}) - G_\bw(\hat{\bmu}^{(T)} \bar{\balpha}^{(T)} ) \vert \le LK \sqrt{\frac{2D\ln (4K/\delta)}{T}} 
\end{align}
In addition, one can show (see (\ref{eq: division bound 1})) that
\begin{align}
  \prob
  &\left[
    \left\|
      \widehat{\bmu}_k^{(T) } \sum_{\tau=1}^T \alpha_k^{(\tau)} 
    -
      \bmu_k \sum_{\tau=1}^T \alpha_k^{(\tau)}
    \right\|
    >
    \sqrt{5DT\ln(2/\delta)}
  \right]
= (DT+1)\delta \notag
\enspace ,
\end{align}
which implies that with probability $1-\delta/2$ we have
\begin{align} \label{eq:aux2}
\left\|
      \widehat{\bmu}_k^{(T) } \bar{\balpha}^{(T)} 
    -
      \bmu_k \bar{\balpha}^{(T)} 
    \right\|
    \le
    \sqrt{\frac{5D\ln(4(DT+1)/\delta)}{T}}
\end{align}
Hence, the difference of regret and pseudo-regret can be upper-bounded with probability at least $1-\delta$ as
\begin{align}
\vert R^{(T)} -  \overline{R}^{(T)} \vert 
& = 
\vert G_{\bw} \left( \bar{\bX}^{(T)} \right) - f( \bar{\balpha}^{(T)} ) \vert \notag \\
& \le 
\vert f^{(T)} (\bar{\balpha}^{(T)})  - f( \bar{\balpha}^{(T)} ) \vert + LK \sqrt{D\frac{2\ln (4K/\delta)}{T}}  \label{eq:aux4} \\
& \le 
L\left(\sqrt{\frac{5D\ln(4(DT+1)/\delta)}{T}} + K \sqrt{\frac{2D\ln (4K/\delta)}{T}} \right) \label{eq:aux5}
\\
& \le 
L\sqrt{\frac{12D\ln(4(DT+1)/\delta)}{T}} \label{eq:aux6}
\end{align}
where (\ref{eq:aux4}) follows from (\ref{eq:aux1}), (\ref{eq:aux5}) follows from (\ref{eq:aux1}) and
(\ref{eq:aux6}) holds because $DT+1 \ge K$.
\end{proof}

\section{Battery control task}\label{app:battery}

Efficient management of electric batteries leads to better performance and longer battery life, which is important for instance for electric vehicles whose range is still limited compared to those with petrol or diesel engines.
An electric battery is composed of several cells whose capacity varies extensively due to inconsistencies in weight and volume of the active material in their individual components, different internal resistance and higher temperatures leading to different aging rates.
As a consequence, the energy output at any instant is different from each cell, which results in different aging rates and ultimately leads to a premature battery failure. 
To address this problem, a control strategy called {\em cell balancing} is utilized, which aims at maintaining a constant energy level --- mainly state-of-charge (SOC) --- in each cell, while controlling for temperature and aging.
Many cell-balancing controllers can be defined, depending on the importance given to the three objectives: SOC, temperature and aging.
The values of those objectives should be balanced between the cells because a balanced use of all cells leads to a longer lasting system. 
Moreover, those objectives values should also be balanced within a cell, because for example, a cell can have higher capacity on a higher temperature, but at the same time it has a higher risk to explode. 


Our battery control task consists in learning and selecting the ``best'' cell balancing controller in an online fashion, so that it can dynamically adapt to the consumption profile and environment, such as outdoor temperature.
In the case of electric cars, this means that the controller needs to be able to adapt to the driving habits of the driver and to the terrain, such as hilly roads or desert roads. 
In this experiment, our goal was more specifically to test our multi-objective online learning algorithms in the battery control task, and verify that our online learning algorithms can indeed find a policy for this control task that leads to a balanced parameter values of the cells.

The battery is modeled using the internal resistance (Rint) model \cite{Johnson02}.
The estimation of SOC is based on the Ampere Hour Counting method \cite{Dambrowski13}.
The variation of temperature in the system is determined according to the dissipative heat loss due to the internal resistance and thermal convection \cite{GaoChenDougal02}.
Cell degradation or aging is a function of temperature, charging/discharging rates and cumulative charge \cite{Tao12}.
Moreover, the battery model is complemented with $10$ different cell-balancing controllers. 
The whole model is implemented in the Matlab/Simulink software package and
can emulate any virtual situation whose electric consumption is described by a time series, which is given as input to the model.
In practice, such a time series would correspond to a short period of battery usage (e.g., a brief segment of a driving session).
For a chosen controller, the output of the model comprises of the objective values of each battery cell at the end of the simulation. 
Note that the output of the battery model is a multivariate random vector since the power demand is randomized, therefore this control task can be readily accommodated into our setup.
In our experiments, the battery consists of $4$ cells, thus $D= 12$ in this case. 
The cell-balancing controllers correspond to the arms, thus $K=10$. 

The online learning task consists of selecting among these controllers/arms so that the final objective values are as balanced as possible. 
The experiment was carried out as follows: in each iteration, the learner selects a controller according to its policy, then the battery model is run with the selected controller by using a random consumption time series. 
At the end of the simulation, the learner receives the objective values of each cell output by the battery model as feedback and updates its policy. 
The goal of the learner is to find a policy over the controllers, which leads to a balanced state of cells in terms of cumulative value. 


The results are shown in Figure~\ref{fig:battery}. 
In this experiment we do not know the means of the arms, but we estimated them based on $100000$ runs of the simulator for each arm. 
These mean estimates were used for computing the optimal policy and the regret. 
We run the \Algo{MO-OGDE} and \Algo{MO-LP} over $100$ repetitions. 
Their average regret exhibits the same trend as in the synthetic experiments: the \Algo{MO-LP} achieved faster convergence. 
The blue lines shows the regret of the pure policies, which always selects the same arm, i.e., the performance of single strategies. 
It is important to mention that the optimal mixed controller has a lower GGI value since the regret of any arm is positive, and more importantly, the two learners converge to optimal mixed policies in terms of regret. 

\newpage

\end{document}